\documentclass{article}

\usepackage{arxiv}

\usepackage[utf8]{inputenc} % allow utf-8 input
\usepackage[T1]{fontenc}    % use 8-bit T1 fonts
\usepackage{hyperref}       % hyperlinks
\usepackage{url}            % simple URL typesetting
\usepackage{booktabs}       % professional-quality tables
\usepackage{amsfonts}       % blackboard math symbols
\usepackage{nicefrac}       % compact symbols for 1/2, etc.
\usepackage{microtype}      % microtypography
\usepackage{lipsum}
\usepackage{xspace}
\usepackage{graphicx}
\usepackage{wrapfig}
\usepackage{amsmath}
\graphicspath{ {./images/} }
\usepackage{xcolor}
\usepackage[noabbrev, capitalise]{cleveref} % no nameinlink
\usepackage{authblk}
\usepackage{float}
\usepackage{ulem}
\usepackage{array}
\usepackage{soul}
\usepackage{subcaption}
\usepackage{listings}
\usepackage{adjustbox}
\usepackage{tabularx}
\usepackage{natbib} 
\usepackage{amsthm}

\newtheorem{theorem}{Theorem}
\newtheorem{assumption}{Assumption}

\newtheorem{lemma}{Lemma}

\usepackage{thmtools,thm-restate}
\usepackage[flushleft]{threeparttable}
\usepackage{booktabs}
\usepackage{makecell}
\usepackage{multirow}

\newcommand{\ABBR}{\text{IN--RIL}\xspace}

\hypersetup{
    colorlinks=true,
    linkcolor=blue,
    filecolor=magenta,      
    urlcolor=blue,
    }

\title{\ABBR: Interleaved Reinforcement and Imitation Learning for Policy Fine-Tuning}
\author{
\begin{tabular}{c}
Dechen Gao$^{1}$, Hang Wang$^{2}$, Hanchu Zhou$^{2}$, Nejib Ammar$^{3}$, Shatadal Mishra$^{3}$, \\
Ahmadreza Moradipari$^{3}$, Iman Soltani$^{4}$, and Junshan Zhang$^{2}$ \\
$^{1}$Department of Computer Science, University of California, Davis \\
$^{2}$Department of Electrical and Computer Engineering, University of California, Davis \\
$^{3}$Toyota InfoTech Labs, Mountain View, CA \\
$^{4}$Department of Mechanical and Aerospace Engineering, University of California, Davis \\
\texttt{\{dcgao,whang,hczhou,isoltani,jazh\}@ucdavis.edu} \\
\texttt{\{nejib.ammar,shatadal.mishra,ahmadreza.moradipari\}@toyota.com}
\end{tabular}
}

\begin{document}
\maketitle

\begin{abstract}

Imitation learning (IL) and reinforcement learning (RL) each offer distinct advantages for robotics policy learning: IL provides stable learning from demonstrations, and  RL promotes generalization through exploration. While existing robot learning approaches using IL-based pre-training followed by RL-based fine-tuning are promising, this two-step learning paradigm often suffers from instability and poor sample efficiency during the RL fine-tuning phase. In this work, we introduce \ABBR, INterleaved Reinforcement learning and Imitation Learning, for policy fine-tuning, which  periodically injects IL updates after multiple RL updates and hence can benefit from the stability of IL and the guidance of expert data for more efficient exploration throughout the entire fine-tuning process. Since IL and RL involve different optimization objectives, we develop gradient separation mechanisms to prevent destructive interference during \ABBR fine-tuning, by separating possibly conflicting gradient updates in orthogonal subspaces. Furthermore, we conduct rigorous analysis, and our findings shed light on why interleaving IL with RL stabilizes learning and improves sample-efficiency. Extensive experiments on 14 robot manipulation and locomotion tasks across 3 benchmarks, including FurnitureBench, OpenAI Gym, and Robomimic, demonstrate that \ABBR can significantly improve sample efficiency and mitigate performance collapse during online finetuning in both long- and short-horizon tasks with either sparse or dense rewards. \ABBR, as a general plug-in compatible with various state-of-the-art RL algorithms, can significantly improve RL fine-tuning, e.g., from 12\% to 88\% with 6.3x improvement in the success rate on Robomimic \texttt{Transport}. Project page: \url{https://github.com/ucd-dare/IN-RIL}.

\end{abstract}

% keywords can be removed

\section{Introduction}

\begin{figure*}[h]
    \centering
    \includegraphics[width=\linewidth]{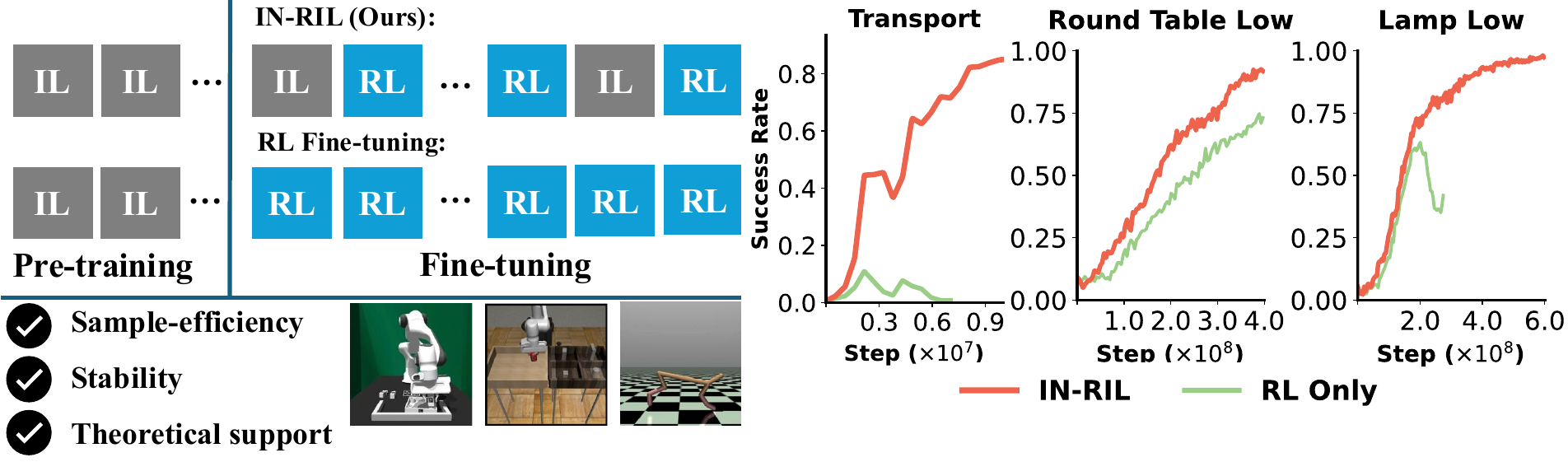}
    \caption{Comparison between \ABBR (interleaved RL/IL) fine-tuning and RL fine-tuning on \texttt{Transport}, \texttt{Round-Table}, and \texttt{Lamp}, which are challenging multi-stage and sparse-reward tasks. Extensive experiments show that IL benefits from expert demonstrations but performance saturates at low success rates; 
    %RL learning from scratch fails (i.e. achieving 0\% success rate), 
    and RL fine-tuning can suffer from stability and poor sample efficiency. \ABBR fine-tuning succeeds to learn and outperforms RL fine-tuning by a significant margin in all tasks.}
    \label{fig:illustration}
\end{figure*}

Recent advances in robotics policy learning have largely been driven by imitation learning (IL) and reinforcement learning (RL)~\cite{chi2023diffusion, fu2024mobile, black2410pi0, wu2023daydreamer}. These two approaches offer complementary strengths for robot learning, yet each comes with limitations when used in isolation. More specifically, in IL (such as behavioral cloning~\cite{shafiullah2022behavior, florence2022implicit}), an agent learns a policy to mimic expert demonstrations, using supervised learning. It is known that while IL provides stable learning dynamics, it faces three critical challenges: the high cost of collecting expert demonstrations~\cite{zhao2024aloha}, limited generalization beyond the demonstration distribution, and vulnerability to compounding errors~\cite{rajeswaran2018learning, ankile2024imitation}. Even small deviations from the demonstration distribution could accumulate and  drastically degrade the  performance. RL approaches, in contrast, learn policies through environmental interaction to maximize accumulated rewards  in a Markov Decision Process (MDP)~\cite{sutton1999policy}. 
Many empirical studies have shown that the
RL approach enables active exploration beyond expert knowledge but often  suffers from instability, sample inefficiency, and hypersensitivity to parameter choices. In particular, these problems are amplified in robotics tasks with sparse rewards and long horizons. For instance, as shown in \Cref{fig:illustration}, the IL method alone yields poor  performance due to the inherent limited coverage of demonstrations, whereas the RL method  struggles to learn effectively through random exploration alone.

% However, collecting demonstrations can be expensive and labor-intensive~\cite{zhao2024aloha}, and performance and generalizability are inherently limited by the diversity of the demonstrations. Additionally, IL suffers from compounding errors and distribution shifts~\cite{rajeswaran2018learning, ankile2024imitation}, in the sense that small deviations from the demonstration distribution can accumulate, drastically reducing performance. On the other hand, in RL, the policy  is learned through interactions with the environment, seeking to maximize accumulated rewards within a Markov Decision Process (MDP)~\cite{sutton1999policy}. While RL allows for active exploration beyond expert knowledge, it frequently suffers from instability, sample inefficiency, and sensitivity to hyper-parameters, especially in robotics tasks with sparse rewards and long horizons. As demonstrated in our empirical results (\Cref{fig:illustration}), pure IL methods yield suboptimal performance given the inherent limitation of the demonstrations, while pure RL methods struggle to learn effectively from random exploration alone. 
%consisting of state-action or sensory observation-action pairs. IL often provides stable and efficient learning dynamics. 

To address the above challenges, recent studies ~\cite{ankile2024imitation, ren2024diffusion, nair2020awac, yuan2024policy} have proposed hybrid approaches that combine IL-based initialization with subsequent RL fine-tuning. While this paradigm leverages the unique strengths of both methods, the critical fine-tuning stage using RL alone continues to face significant challenges that limit its effectiveness. Specifically, RL fine-tuning  often suffers from performance collapse, instability, and poor sample efficiency~\cite{nakamoto2023cal, nair2020awac, rajeswaran2018learning}.  Existing approaches may improve fine-tuning by adding demonstrations into replay buffers~\cite{ball2023efficient, song2022hybrid}, which requires reward annotations and complex sampling strategies~\cite{ball2023efficient, hu2023imitation}, or add regularization terms to constrain policy drift~\cite{rajeswaran2018learning, haldar2023watch}, which demands careful hyperparameter tuning. These limitations presents a fundamental question that we aim to address in this work:

\begin{center}\vspace{-0.07in}
    \textit{
How to synergize the stability of IL with the exploration benefits of RL \\ for efficient policy fine-tuning?
}    
\end{center}

\begin{wrapfigure}{r}{0.40\textwidth}\vspace{-0.2in}
\centering
    \includegraphics[width=\linewidth]{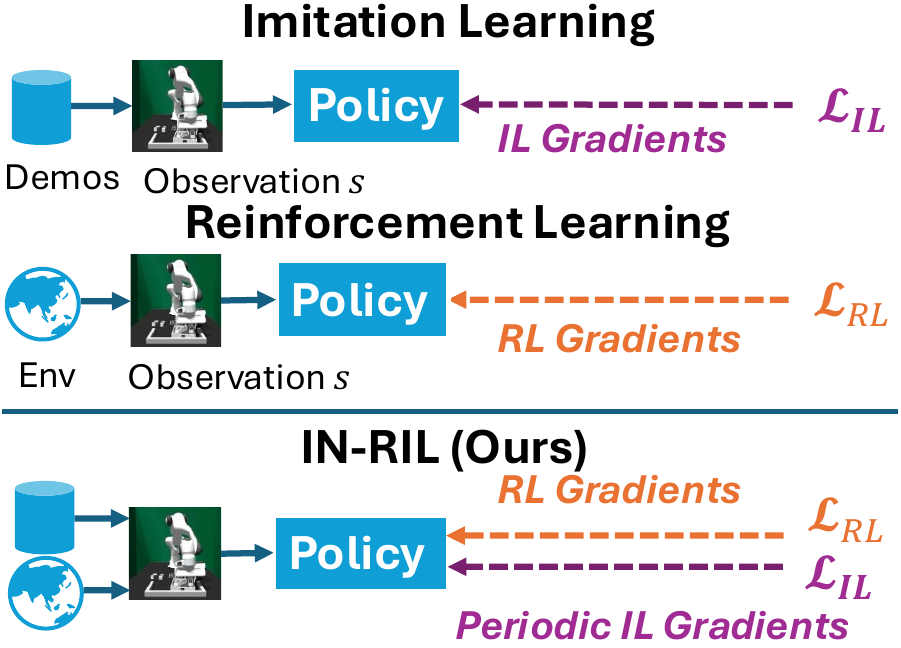}
    \caption{An illustration of \ABBR which updates the policy network with both IL and RL objectives.}
    \label{fig:rl_il_inril_illustration} 
\end{wrapfigure}\vspace{-0in}

Thus motivated, we propose \ABBR (INterleaved Reinforcement and Imitation Learning) fine-tuning that can cleverly exploit demonstration data throughout the fine-tuning process. As illustrated in \Cref{fig:rl_il_inril_illustration}, \ABBR  integrates IL updates with RL fine-tuning by periodically inserting one IL update after every few RL updates. As shown clearly in \Cref{fig:illustration}, \ABBR fine-tuning outperforms RL fine-tuning by a significant margin in the challenging long-horizon and sparse-reward tasks.  We summarize our key insight for IN-RIL as follows.

As illustrated in \Cref{fig:interleaving_mechanism},  IL and RL objectives create different non-convex optimization landscapes, which are often not aligned. \textbf{Both IL and RL have multiple local minima/optima, indicating that when fine-tuning  using RL or IL alone could be trapped at a local minimum.}
By interleaving IL and RL updates during fine-tuning,  \ABBR can help RL to jump out of a lower reward neighborhood towards a higher reward neighborhood, and in the meanwhile RL updates can help to move IL out its local minima in its loss landscape to another local minima with lower losses.

Given that IL and RL involve different optimization landscapes, we caution that it is of critical importance to avoid destructive interference between their respective gradient updates in \ABBR. To address this challenge, we devise gradient separation mechanisms that effectively combine learning signals while preventing conflicts between these different objectives. In particular, we have developed two implementation approaches: (1) gradient surgery~\cite{sener2018multi, jacobian_descent}, which mitigates interference through gradient projection techniques; and (2) network separation, which isolates RL gradients in a residual policy while the base policy continues to leverage IL. Both methods effectively separate IL and RL gradient updates in different subspaces to prevent destructive interactions. It is worth noting that \ABBR is algorithm-agnostic and can serve as a plug-in to existing RL frameworks, as demonstrated through our integration with state-of-the-art methods including DPPO~\cite{ren2024diffusion}, IDQL~\cite{hansen2023idql}, residual PPO~\cite{ankile2024imitation, yuan2024policy}, covering both on-policy and off-policy approaches.

\begin{figure*}[t]
    \centering
    \includegraphics[width=\linewidth]{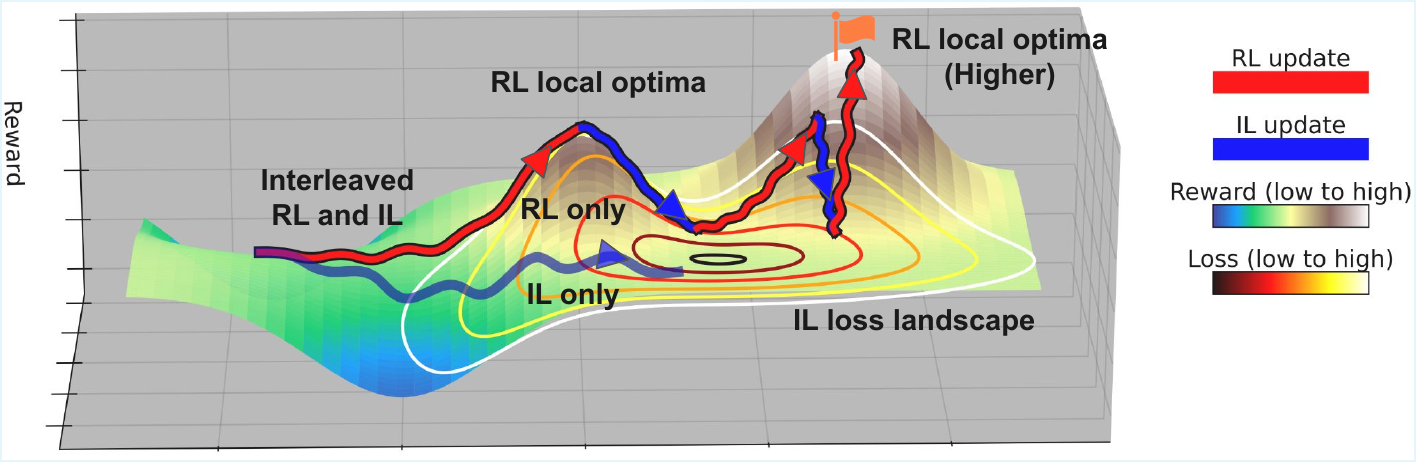}
    \caption{Optimization landscapes for \ABBR. The IL loss landscape, represented by the 3D surface topology and its corresponding contour lines (where each contour connects points of equal IL loss value); and the landscape of RL rewards  (or negative of the loss), represented by the color gradient mapped onto the surface (where the blue-to-white spectrum indicates low-to-high reward values as shown in the legend). IL updates drive the policy toward regions with lower losses, while RL updates steer toward higher rewards. Both optimization processes are stochastic and non-convex with multiple local optima. When using either RL or IL alone, training often converges to suboptimal solutions (as shown in the ``IL only'' and ``RL only'' trajectories). In contrast, our IN-RIL approach enables each objective to help escape the other's local optima: periodic IL updates help RL escape lower-reward regions toward higher-reward neighborhoods, while RL updates help IL traverse between different local minima in the loss landscape. }
    \label{fig:interleaving_mechanism} \vspace{-0in}
\end{figure*} 

{\bf{Summary of Contributions}}  In summary, our work makes the following contributions:
\begin{itemize}
    \item \textbf{\ABBR.} We introduce \ABBR, a fine-tuning approach that periodically interleaves imitation learning updates with reinforcement learning updates, addressing the limitations of conventional two-step methods. Intuitively,  by periodically inserting one IL iteration after every few RL iterations, \ABBR synergizes the stability of IL using expert demonstrations with the exploration capabilities of RL throughout the fine-tuning process.

\item \textbf{Gradient Separation Mechanisms.} Given that IL and RL involve different optimization objectives, we develop gradient separation mechanisms to prevent destructive interference during interleaved training. Our methods effectively separate possibly conflicting gradient updates in orthogonal subspaces, reaping  the benefits of both approaches while minimizing conflicts during fine-tuning. More specifically, we propose two separation techniques: 1) gradient surgery, where RL and IL gradients are projected into independent subspaces to mitigate conflicts; 2) network separation, where \ABBR introduces a residual policy network updated by RL gradients, while base policy is not updated by RL gradients, and therefore, avoids the conflicts.

\item \textbf{Analytic Foundation.} We carry out analysis to characterize the foundational reason why \ABBR outperforms conventional RL fine-tuning in both stability and sample efficiency, offering insights into the optimal interleaving ratio for maximizing performance across diverse robotics tasks.

\item \textbf{Algorithm-Agnostic Design with Comprehensive Validation.} We demonstrate \ABBR's effectiveness as a general plug-in compatible with state-of-the-art RL algorithms, including on-policy methods (DPPO, residual PPO) and off-policy approaches (IDQL). Through extensive experiments on 14 challenging robotics tasks across FurnitureBench~\cite{heo2023furniturebench}, Robomimic~\cite{mandlekar2021matters}, and OpenAI Gym~\cite{brockman2016openai}, we show that \ABBR can substantially improve performance — e.g., boosting success rates to 88\%, when integrated with RL algorithms that originally yield only 12\% success rates on Robomimic \texttt{Transport}. Our evaluations span both long-horizon and short-horizon scenarios with sparse and dense rewards, demonstrating the broad applicability of our approach.
\end{itemize}

\section{Related Work}
\label{sec:related}

\textbf{Robotics Policy Learning and Fine-Tuning.} Imitation learning (IL)~\cite{brohan2022rt, kim2024openvla, chi2023diffusion, fu2024mobile, lee2024interact, florence2022implicit} and reinforcement learning (RL)~\cite{kalashnikov2018scalable, han2023survey, hafner2023mastering, ren2024diffusion, wu2023daydreamer, gao2024cardreamer, ankile2024imitation} have been widely studied in robotics. IL assumes access to expert demonstrations and is generally more stable to train~\cite{chi2023diffusion, shafiullah2022behavior}, but it suffers from distribution shifts and often fails to generalize beyond demonstrations~\cite{rajeswaran2018learning}. In addition, collecting high-quality expert data can be labor-intensive and costly, sometimes requiring hundreds or even thousands of demonstrations per task~\cite{zhao2024aloha} through teleoperation ~\cite{fu2024mobile}, or VR equipments~\cite{chuang2024active}. On the other hand, RL enables agents to explore and self-improve, potentially overcoming IL limitations of labor-intensive data collection and generalization. However, RL is notoriously sample-inefficient~\cite{song2022hybrid}, especially for long-horizon tasks with sparse rewards~\cite{gupta2019relay}, where agents may easily fail to explore and learn. Recent works have proposed combining IL and RL in a two-stage pipeline: IL is first used to pre-train a reasonable policy to warm-start the RL process, followed by RL fine-tuning to further improve generalization via exploration~\cite{ren2024diffusion, ankile2024imitation, hu2023imitation}. The same paradigm was also applied to LLM fine-tuning~\cite{guo2025deepseek}. In this work, we move beyond the two-stage paradigm, and show that the data used for pre-training, even after pre-training plateaus, is still valuable in improving sample-efficiency and stability of RL fine-tuning.

\textbf{RL with Expert Demonstrations.} Recent works have explored leveraging offline data for training RL policies. ROT~\cite{haldar2023watch, rajeswaran2018learning} introduces a regularization term to RL objectives to keep the policy close to expert behaviors, which, however, requires careful balancing between RL objectives and the regularization term. AWAC~\cite{nair2020awac}, Hy--Q~\cite{song2022hybrid}, IBRL~\cite{hu2023imitation}, RLPD~\cite{ball2023efficient}, Cal-QL~\cite{nakamoto2023cal} add expert data with rewards to a replay buffer and perform off-policy updates during online learning. However, it can be infeasible to perform off-policy RL updates on expert demonstrations since reward annotations are not always available. Furthermore, sampling strategy is shown to be crucial for off-policy updates when there are both demonstration data and RL-collected data~\cite{hu2023imitation, ball2023efficient}. In contrast, \ABBR does not introduce explicit regularization terms which rely on delicate loss balancing, and can over-regularize the policy and damage performance. \ABBR does not assume availability of rewards in IL data, or require sampling strategies to balance learning from offline and online data. Instead, it treats IL and RL as complementary optimization processes and interleaves them during fine-tuning without modifying the RL algorithm itself. This makes \ABBR broadly applicable to both on-policy and off-policy RL methods.

%\section{Interleaved R/I}
\section{\ABBR: Interleaved RL and IL for Efficient Policy Finetuning}

In this section, we provide a theoretical analysis of \ABBR, aiming to answer two key questions: (1) what is the optimal interleaving ratio of RL updates to IL updates that balances learning stability and performance improvement, and (2) How much reduction in  iteration complexity can be achieved by our proposed \ABBR approach? We derive conditions under which \ABBR achieves superior sample efficiency and faster convergence to target performance levels. These theoretical results not only justify our algorithmic design choices but also provide practical guidance for adapting the interleaving ratio based on gradient alignment during training.

%we use diffusion policies as a concrete and important example as diffusion-based approaches have recently shown remarkable success in robotics applications, notably in frameworks like Diffusion Policy Policy Optimization (DPPO) \cite{ren2024diffusion}. Our theoretical framework offers insights into why \ABBR succeed and provides performance guarantees that apply beyond just diffusion policies.

\paragraph{Markov Decision Process.} We consider a Markov Decision Process (MDP) defined by the tuple $\mathcal{M} = (\mathcal{S}, \mathcal{A}, P, r, \gamma, \rho_0)$, where $\mathcal{S}$ is the state space, $\mathcal{A}$ is the action space, $P: \mathcal{S} \times \mathcal{A} \times \mathcal{S} \to [0, 1]$ is the transition probability function, $r: \mathcal{S} \times \mathcal{A} \to \mathbb{R}$ is the reward function, $\gamma \in [0, 1)$ is the discount factor, and $\rho_0$ is the initial state distribution. A policy $\pi: S \rightarrow \Delta(A)$ maps states to probability distributions over actions. The action-value function, or Q-function, for a policy $\pi$ is defined as $Q^{\pi}(s, a) = \mathbb{E}{\pi}\left[\sum{t=0}^{\infty} \gamma^t r(s_t, a_t) | s_0 = s, a_0 = a\right]$, representing the expected cumulative discounted reward when taking action $a$ in state $s$ and following policy $\pi$ thereafter. The objective in RL is to find a policy that maximizes the expected Q-value: $\mathbb{E}_{s \sim \rho_0, a \sim \pi(·|s)}[Q^{\pi}(s, a)]$.

\paragraph{Pre-Training.} We consider a parametric policy $\pi_\theta: \mathcal{S} \rightarrow \Delta(\mathcal{A})$ that maps states to distributions over actions. We employ a direct policy representation where $\pi_\theta(a|s)$ gives the probability (or probability density) of taking action $a$ in state $s$. This formulation allows for direct  optimization through gradient-based methods while maintaining sufficient expressivity for complex robotic control tasks. During pre-training, we use behavior cloning to learn a policy that imitates expert demonstrations $\mathcal{D}_{\text{exp}} = \{\tau_1, \tau_2, \ldots, \tau_N\}$, where each trajectory $\tau_i = \{(s_1,a_1),\ldots,(s_T,a_T)\}$ contains state-action pairs. The objective is to maximize the likelihood of expert actions given the corresponding states:
\begin{align}
\mathcal{L}_{\mathrm{IL}}(\theta) = \mathbb{E}_{(s,a^) \sim \mathcal{D}_{\text{exp}}}[-\log\pi_\theta(a \vert s)],
\end{align}
where $a^*$ represents the expert action. This negative log-likelihood objective encourages the policy to assign high probability to actions demonstrated by experts in the same states. We then obtain a warm-start policy $\pi_{0} = \arg\min_{\pi_\theta} \mathcal{L}_{\mathrm{IL}}(\theta)$ that serves as the initialization for subsequent fine-tuning. This pre-training approach allows the policy to capture the basic structure of the task before reinforcement learning is applied to further optimize performance. After obtaining a policy via imitation learning during pre-training, we proceed to the finetuning phase where we optimize the policy. In our analysis, we compare two distinct finetuning approaches, RL Finetuning and our proposed \ABBR. 
\paragraph{RL Finetuning.} After pretraining, RL finetuning directly optimizes policy parameters to maximize the expected Q-value as defined earlier, through gradient updates of the form:
\begin{align*}
\theta_{t+1} = \theta_t - \alpha_{\mathrm{RL}} \nabla_\theta \mathcal{L}_{\mathrm{RL}}(\theta_t)
\end{align*}
where $\alpha_{\mathrm{RL}}$ is the learning rate and $\mathcal{L}_{\mathrm{RL}}(\theta) = -\mathbb{E}_{s \sim d^{\pi_\theta}}[Q^{\pi_\theta}(s, \pi_\theta(s))]$ is defined as the loss function, which represents the negative of the expected Q-value under the current policy's state distribution $d^{\pi_\theta}$. This formulation directly connects to our optimization objective of maximizing $\mathbb{E}_{s \sim \rho_0, a \sim \pi(·|s)}[Q^{\pi}(s, a)]$, but accounts for the evolving state distribution as the policy improves. While this approach aims to maximize the overall reward, it often suffers from instability and poor sample efficiency, particularly when finetuning complex models like diffusion policies.

\paragraph{\ABBR.}  As depicted in \Cref{fig:illustration}, the proposed \ABBR systematically alternates between IL and RL updates:
\begin{align*}
\theta_{t+\frac{1}{1+m(t)}} &= \theta_{t} - \alpha_{\mathrm{IL}} \nabla_\theta \mathcal{L}_{\mathrm{IL}}(\theta_{t}) \\
\theta_{t+\frac{1+j}{1+m(t)}} &= \theta_{t+\frac{j}{1+m(t)}} - \alpha_{\mathrm{RL}} \nabla_\theta \mathcal{L}_{\mathrm{RL}}(\theta_{t+\frac{j}{1+m(t)}}), \quad j \in \{1,\ldots,m(t)\}
\end{align*}
where $m(t)$ represents the iteration-dependent number of RL updates performed after each IL update. The IL updates help maintain the desirable behaviors from pre-training while providing regularization, and the RL updates improve performance on the target task.

Our analysis uses standard assumptions regarding the pretraining performance, data coverage, smoothness properties of the loss functions, and gradient estimation quality. Specifically, we assume that: (1) the initial policy obtained by pretraining results in a training loss within a bounded distance from  the IL objective; (2) the expert demonstration dataset provides reasonably sufficient coverage of the relevant state space for the target task; (3) both the IL  and RL objectives satisfy smoothness conditions; and (4) the stochastic gradient estimates for both objectives have bounded variance that decreases proportionally with batch size. The formal statements of these assumptions (Assumptions \ref{asu:pre}-\ref{asu:variance}) and their implications are provided in Appendix A.

Next, we introduce the assumptions on the  geometric relationship between the gradients of the IL and RL objectives in \Cref{asu:gradient}. In particular, we use the parameter $\rho(t)$ to capture the cosine similarity between these gradients, with positive values indicating opposing gradients and negative values indicating aligned gradients. Such assumption has been commonly used in multi-objective optimization \cite{sener2018multi,desideri2012multiple}.

\begin{assumption}[Gradient Relationship]
In the finetuning regime, the gradients of IL and RL objectives exhibit the following relationship:
\begin{align*}
    \langle\nabla_\theta\mathcal{L}_{\mathrm{IL}}(\theta_t), \nabla_\theta\mathcal{L}_{\mathrm{RL}}(\theta_t)\rangle = -\rho(t)\|\nabla_\theta\mathcal{L}_{\mathrm{IL}}(\theta_t)\| \cdot \|\nabla_\theta\mathcal{L}_{\mathrm{RL}}(\theta_t)\|
\end{align*}
where $\rho(t) \in [-1, 1]$ represents the time-varying relationship between gradients, with positive values indicating opposition (negative cosine similarity) and negative values indicating alignment (positive cosine similarity). \label{asu:gradient}
\end{assumption}

Based on these assumptions, we establish the following key results on the optimal ratio of RL updates to IL updates in the proposed \ABBR. This ratio is crucial for balancing the stability provided by IL updates with the performance improvements offered by RL updates.

\begin{theorem}[Optimal Interleaving Ratio]\label{thm:optimal}
Under Assumptions \ref{asu:gradient}-\ref{asu:variance}, at iteration $t$, the optimal ratio $m(t)$ for \ABBR satisfies $ m_{\text{opt}}(t) \geq 1$.
\end{theorem}

\Cref{thm:optimal} provides a principled formula for adapting the interleaving ratio throughout training based on current gradient information. The optimal ratio $m_{\text{opt}}(t)$ increases when gradients strongly ``oppose'' each other ($\rho(t) > 0$) and decreases when they are more aligned ($\rho(t) > 0$), reflecting the intuition that more RL updates are needed to make progress when IL updates work against the RL objective. This result suggests that monitoring gradient alignment during training can lead to more efficient optimization strategies compared to using a fixed interleaving ratio. Given this optimal ratio, we next quantify exactly how much more efficient \ABBR can be compared to RL-only approaches. Denote  $\Delta_{\mathrm{IL-RL}} =- \sum_{t=0}^{T-1}\frac{c_{\mathrm{IL}}\rho(t)}{L_{\mathrm{IL}}}\|\nabla\mathcal{L}_{\mathrm{IL}}(\theta_t)\| \cdot \|\nabla\mathcal{L}_{\mathrm{RL}}(\theta_t)\| - \frac{c^2_{\mathrm{IL}}\sigma^2_{\mathrm{IL}}T}{2L_{\mathrm{IL}}N_{\mathrm{IL}}}$.Then we have:

\begin{theorem}[Iteration  Complexity of \ABBR]
Under Assumptions \ref{asu:gradient}-\ref{asu:variance}, for a fixed computational budget of $T$ total updates, \ABBR with $m > 1$ and $\Delta_{\mathrm{IL-RL}} > \frac{L_{\mathrm{RL}}(\mathcal{L}_{\mathrm{RL}}(\theta_0) - \mathcal{L}^*_{\mathrm{RL}})}{m+1}$ requires fewer iterations to reach a target accuracy $\epsilon$ than RL-only finetuning, i.e., $\frac{T_{\text{RL-only}}}{T_{\text{\ABBR}}} > 1 $.

\label{thm:efficiency}
\end{theorem}

Theorem \ref{thm:efficiency} establishes the conditions under which \ABBR achieves superior  efficiency compared to RL-only finetuning. Specifically, when the regularization benefit $\Delta_{\mathrm{IL-RL}}$ exceeds the threshold $\frac{L_{\mathrm{RL}}(\mathcal{L}_{\mathrm{RL}}(\theta_0) - \mathcal{L}^*_{\mathrm{RL}})}{m+1}$, \ABBR requires fewer total updates to reach the same performance level. This threshold depends critically on the interleaving ratio $m$, with higher values of $m$ reducing the required regularization benefit for efficiency gain. Intuitively, this means that when the stabilizing effect of periodically revisiting the demonstration data is sufficiently strong, and the interleaving ratio is properly set, \ABBR can achieve the same performance with fewer total updates. This theoretical guarantee aligns with our empirical observations across multiple robotics tasks, where \ABBR consistently demonstrates faster convergence and higher sample efficiency than pure RL approaches. The result provides formal justification for the \ABBR and offers practical guidance for setting the interleaving ratio based on task characteristics.

% \Cref{thm:efficiency} establishes the computational advantage of \ABBR over RL-only approaches, demonstrating that it requires asymptotically fewer iterations by a factor of $O(\bar{m})$. This result shows that the efficiency gain scales directly with how many RL updates we perform per IL update, indicating that higher interleaving ratios lead to greater computational savings when a positive regularization benefit exists. The finding has significant practical implications, suggesting that the additional computational overhead of performing IL updates is far outweighed by the acceleration in learning progress. This theoretical result aligns with empirical observations in recent work, where \ABBR converges more rapidly to high-performance policies than pure RL methods, especially in complex tasks requiring both behavioral consistency and performance optimization.

\section{Experiments}

Based on the above analysis, we further conduct a comprehensive empirical evaluation to address two key questions: 1) What are the benefits of \ABBR compared to RL fine-tuning? 2) What is the impact of the interleaving ratio $m$ on the  performance? To this end, we evaluate \ABBR on 14 different tasks across three widely adopted robotics benchmarks, including FurnitureBench~\cite{heo2023furniturebench}, OpenAI Gym~\cite{brockman2016openai}, and Robomimic~\cite{mandlekar2021matters}. These benchmarks represent a diverse spectrum of robotics challenges, encompassing both locomotion and manipulation tasks with varying reward structures (sparse and dense) and time horizons (short and long).

\textbf{Robomimic~\cite{mandlekar2021matters}.} We evaluate \ABBR on four robot manipulation tasks from Robomimic: \texttt{Lift}, \texttt{Can}, \texttt{Square}, and \texttt{Transport}. Among these, \texttt{Square} and \texttt{Transport} are particularly challenging for RL agents~\cite{ren2024diffusion}. All tasks feature sparse rewards upon successful completion, with each task providing 300 demonstrations. For \texttt{Transport} and \texttt{Lift}, we specifically use noisy multi-human demonstration data to test robustness. Notably, when coupled with \ABBR, IDQL, one of the best off-policy fine-tuning algorithms, achieves only $12\%$ success rates on \texttt{Transport}, while \ABBR boosts it to $88\%$, a $6.3\times$ improvement.

% All tasks use sparse rewards upon successful completion. Each task has 300 demonstrations. Notably, we use noisy multi-human data for \texttt{Transport} and \texttt{Lift}. \textbf{We will show that IDQL, one of the best off-policy fine-tuning algorithm yet failing \texttt{Transport} with $12\%$ success rates, when coupled with \ABBR, achieves $87\%$ success rates with 6.25x improvement.}  % DPPO~\cite{ren2024diffusion} was reported to be the only existing approach that achieves over $50\%$ success rate on both state-based and pixel-based \texttt{Transport}.

\textbf{FurnitureBench~\cite{heo2023furniturebench}.} FurnitureBench presents the most challenging tasks in our experiments, featuring long-horizon, multi-stage manipulation tasks with sparse rewards. We include three assembly tasks: \texttt{One-Leg}, \texttt{Lamp}, and \texttt{Round-Table}, each with both \texttt{Low} and \texttt{Med} randomness settings for state distributions. Each task includes 50 human demonstrations and provides sparse stage-completion rewards. We additionally incorporate two tasks from ResiP~\cite{ankile2024imitation}: \texttt{Mug-Rack} and \texttt{Peg-in-Hole}, resulting in a total of 7 tasks when accounting for randomness variants.

% FurnitureBench is the most challenging benchmark among three benchmarks selected, since it features long-horizon and multi-stage manipulation tasks with sparse rewards. We include three assembly tasks: \texttt{One-Leg}, \texttt{Lamp}, and \texttt{Round-Table}. FurnitureBench provides \texttt{Low} and \texttt{Med} randomness for state distributions. Each task is equipped with 50 human demonstrations. All tasks have sparse stage-completion rewards. Besides three assembly tasks in FurnitureBench, we include two more tasks provided by ResiP~\cite{ankile2024imitation}, \texttt{Mug-Rack} and \texttt{Peg-in-Hole}. We have 7 tasks, including randomness variants, in total.

\textbf{OpenAI Gym~\cite{brockman2016openai}.}  To evaluate performance on dense-reward tasks, we include three classic locomotion benchmarks: \texttt{Hopper} (v2), \texttt{Walker2D} (v2), and \texttt{HalfCheetah} (v2). For these tasks, we utilize the medium-level imitation datasets from D4RL~\cite{fu2020d4rl}.

\subsection{Training}

We evaluate \ABBR with multiple policy parameterizations for pre-training, including diffusion policy (DP)\cite{chi2023diffusion} and Gaussian policy\cite{sutton1999policy}, both of which are widely adopted in recent IL and RL literature~\cite{chi2023diffusion, zhao2024aloha, ren2024diffusion, ankile2024imitation}. Particularly, DP has consistently demonstrated superior performance across robotics tasks in both pre-training~\cite{chi2023diffusion} (see Table~\ref{tab:pretrain_results_furniture}) and fine-tuning~\cite{ren2024diffusion}. We employ action chunking~\cite{fu2024mobile} to enhance temporal consistency. For fine-tuning, we select three state-of-the-art RL algorithms spanning both on-policy and off-policy approaches: 1) PPO~\cite{schulman2017proximal, ankile2024imitation, yuan2024policy}, a widely used on-policy algorithm; 2) DPPO~\cite{ren2024diffusion}, an on-policy, policy gradient-based RL algorithm; and 3) IDQL~\cite{florence2022implicit}, an off-policy, Q-learning-based RL algorithm. DPPO and IDQL are both DP-based RL algorithms. This diverse selection enables us to comprehensively evaluate \ABBR's effectiveness across different RL algorithms and policy parameterizations.

\paragraph{Pre-Training.} Taking FurnitureBench as an example, we pre-train different policy parameterizations using 50 demonstrations with IL until convergence. As shown in Table~\ref{tab:pretrain_results_furniture}, Gaussian policy without action chunking fails entirely on these challenging multi-stage sparse-reward tasks, while Gaussian policy with action chunking achieves limited success. DP demonstrates the strongest overall performance across all tasks in FurnitureBench, Robomimic, and Gym. However, even DP pre-training remains sub-optimal, with 3 tasks showing below $5\%$ success rates after loss plateau, primarily due to limited dataset coverage. % Pre-training hyperparameters and detailed results for other benchmarks are provided in the Appendix.

\begin{table}[h!]
\centering
\small
\setlength{\tabcolsep}{3pt}
\begin{tabular}{l|lcc|cc|c|c|c|c}
\toprule
 & \textbf{Policy Parameterization} & \multicolumn{2}{c|}{\texttt{OneLeg}} & \multicolumn{2}{c|}{\texttt{Lamp}} & \texttt{RoundTable} & \texttt{MugRack} & \texttt{PegInHole} \\
& & \texttt{Low} & \texttt{Med} & \texttt{Low} & \texttt{Med} & & & \\
\midrule
\multirow{3}{*}{BC} 
    & Gaussian w/ Action Chunking & 0.38 & 0.17 & 0.07 & 0.02 & 0.01 & 0.14 & 0.02 \\
    & Gaussian w/o Action Chunking & 0.0 & 0.0 & 0.0 & 0.0 & 0.0 & 0.0 & 0.0 \\
    & DP     & \textbf{0.47} & \textbf{0.28} & 0.05 & \textbf{0.1} & \textbf{0.10} & \textbf{0.19} & \textbf{3} \\
    % & DP-MLP     & 47 & 28 & 5 & 1 & 1 & 12 & 19 & 3 \\
\bottomrule
\end{tabular}
\caption{Success rates across FurnitureBench tasks ~\cite{ankile2024imitation, heo2023furniturebench} using pre-trained policies.}
\label{tab:pretrain_results_furniture}
\end{table}

\paragraph{Fine-Tuning.} While DP yields the best pre-training performance, fine-tuning DP with conventional RL algorithms presents significant challenges and can lead to failure~\cite{ren2024diffusion, yang2023policy}. We consider two strategies for RL fine-tuning: 1) \textit{Full network fine-tuning}, where we use specialized RL algorithms (DPPO and IDQL) to fine-tune the entire pre-trained DP network; and 2) \textit{Residual policy fine-tuning}, where we introduce an additional Gaussian policy as a residual policy on top of the pre-trained DP (base) policy. The residual policy, implemented as an MLP network, is fine-tuned with conventional RL (PPO)~\cite{schulman2017proximal, ankile2024imitation} while the base policy is updated solely with IL. The residual policy learns to adjust the base policy's actions at each time step. For each task, we fine-tune the pre-trained DP checkpoint with the highest success rate (or reward) using \ABBR, and compare against RL-only fine-tuning. While our theory suggests an adaptive ratio $m(t)$, we use a constant value of $m$ throughout training for simplicity. Based on our results, values of $m$ between 5 and 15 work well across most tasks, balancing performance improvement with policy stability. We conduct a detailed ablation study on the impact of different $m$ values in Section~\ref{sec:ablation}. % Implementation details for all three algorithms under the \ABBR framework are provided in the Appendix.

 \begin{wrapfigure}{r}{0.27\textwidth} \vspace{-0.35in}
    \centering
    \includegraphics[width=0.7\linewidth]{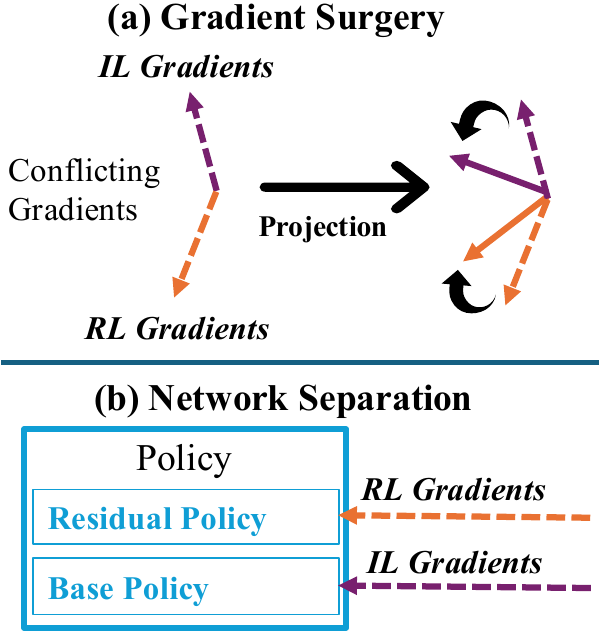}
    \caption{An illustration of the two gradient separation mechanisms: a) gradient surgery, and b) network separation.} \label{fig:gradient_separation_illustration}
\end{wrapfigure}

\paragraph{Separation of RL and IL gradients for \ABBR.} RL and IL each operate within distinct optimization landscapes, meaning a policy that is optimal from an RL perspective (high rewards) may not be optimal from an IL perspective (low BC losses), and vice versa. Directly updating a single network with these potentially conflicting objectives can degrade policy performance (as demonstrated in our ablation study in Section~\ref{sec:ablation_gradient_separation}).

To address this challenge, as illustrated in \Cref{fig:gradient_separation_illustration}, we introduce two gradient separation techniques that prevent interference between RL and IL objectives. The first technique, 1) \textit{gradient surgery}, projects each gradient onto the dual cone~\cite{jacobian_descent}, ensuring that updates benefit both individual objectives. The second technique, 2) \textit{network separation}, is naturally integrated with the residual RL fine-tuning strategy. This approach allocates IL gradients to the base policy while RL gradients update the residual policy, effectively mitigating interference.

\subsection{\ABBR vs. RL Fine-tuning}
We demonstrate that \ABBR can enhance the performance of state-of-the-art RL fine-tuning algorithms across diverse robotic tasks. For each benchmark, we select the best-performing RL algorithms according to recent literature: DPPO~\cite{ren2024diffusion} and IDQL~\cite{hansen2023idql} for Robomimic and Gym tasks, and residual PPO~\cite{ankile2024imitation} for FurnitureBench. Our comprehensive evaluation reveals that \ABBR consistently outperforms these top-performing algorithms in terms of sample efficiency, stability, and final performance. The results for Robomimic and Gym tasks using DPPO and IDQL are presented in Figure~\ref{fig:dppo_perf} and Figure~\ref{fig:idql_perf}, respectively. FurnitureBench results are shown in ~\Cref{fig:residual_furniture_perf}.

We also compare \ABBR with other RL fine-tuning algorithms in ~\Cref{tab:robomimic_gym} and ~\Cref{tab:furniture_residual}. The other baselines include DPPO augmented by BC loss regularization~\cite{rajeswaran2018learning} (denoted as ``BC Loss" in the table), AWC~\cite{peng2019advantage, ren2024diffusion}, and DIPO~\cite{yang2023policy}.

\begin{figure}[H]
    \centering
    \includegraphics[width=\linewidth]{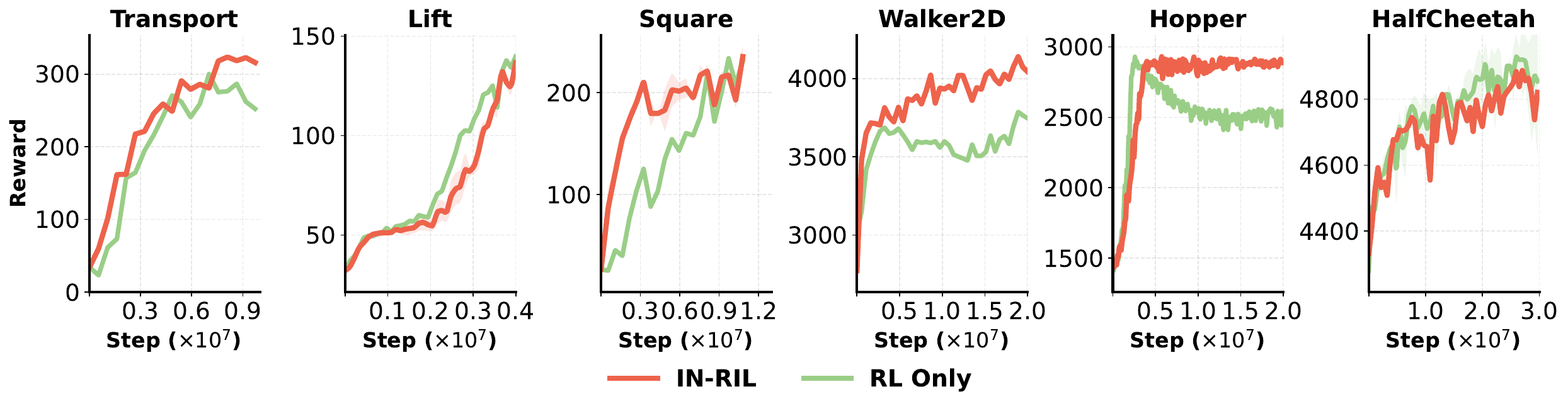}
    \caption{Comparing \ABBR with RL fine-tuning on Robomimic and Gym using DPPO.}
    \label{fig:dppo_perf}
\end{figure}

\begin{figure}[H]
    \centering
    \includegraphics[width=\linewidth]{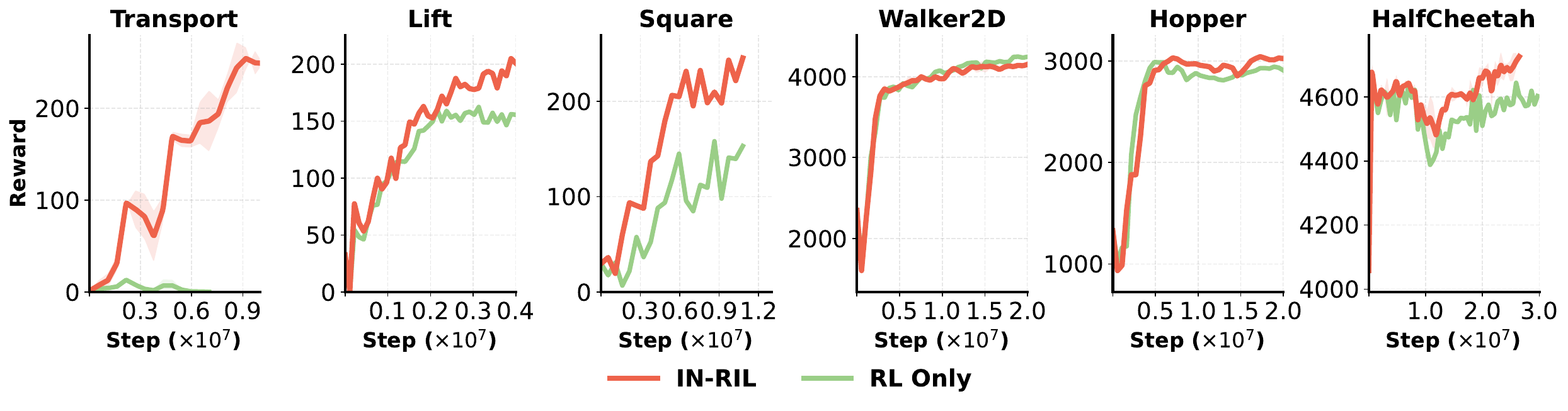}
    \caption{Comparing \ABBR with RL fine-tuning on Robomimic and Gym using IDQL.}
    \label{fig:idql_perf}
\end{figure}

\begin{table}[H]
\centering

\begin{tabular}{l|cc|cc|ccc}
\toprule
Task & IN-RIL (DPPO) & DPPO & IN-RIL (IDQL) & IDQL & BC Loss & DIPO & AWR \\
\midrule
Transport & \textit{\textbf{0.95}} & 0.89 & \textbf{0.88} & 0.12 & 0.41 & 0.16 & 0.16 \\
Can & \textit{\textbf{1.00}} & \textit{\textbf{1.00}} & 0.98 & \textit{\textbf{1.00}} & 0.96 & 0.94 & 0.65 \\
Lift & \textit{\textbf{1.00}} & \textit{\textbf{1.00}} & \textit{\textbf{1.00}} & 1.00 & 0.98 & 0.97 & 0.99 \\
Square & \textbf{0.91} & 0.90 & \textit{\textbf{0.98}} & 0.80 & 0.64 & 0.59 & 0.51 \\
\midrule
Walker2D & \textbf{4139} & 3786 & 4186 & \textbf{4248} & 3457 & 3715 & \textit{4250} \\
Hopper & \textbf{2930} & 2929 & \textit{\textbf{3042}} & 2988 & 2896 & 2938 & 1427 \\
HalfCheetah & 4887 & \textit{\textbf{5011}} & \textbf{4742} & 4671 & 4532 & 4644 & 4611 \\
\bottomrule
\end{tabular}
\caption{Performance comparison for all fine-tuning methods on Robomimic (using success rates) and Gym tasks (using rewards). Bold values indicate the best in the DPPO group, or IDQL group. Italic values indicate the overall best across all methods.}
\label{tab:robomimic_gym}
\end{table}

Figure~\ref{fig:dppo_perf} and Figure~\ref{fig:idql_perf} show that \ABBR consistently improves upon both DPPO and IDQL across manipulation and locomotion tasks. Notably, on the two most challenging Robomimic tasks, \texttt{Transport} and \texttt{Square}~\cite{ren2024diffusion}, \ABBR substantially boosts performance of both DPPO and IDQL. The gains are especially prominent when combined with IDQL, where RL-only fine-tuning fails on \texttt{Transport} with $12\%$ success rates, while \ABBR successfully solves the task and achieves $88\%$ success rates, as shown in ~\Cref{fig:idql_perf} and ~\Cref{tab:robomimic_gym}; on \texttt{Square}, \ABBR improves IDQL by $22.5\%$ in success rates; and reduces $62\%$ environment steps needed for DPPO to converge in ~\Cref{fig:dppo_perf}. This highlights the crucial role of IL guidance for RL exploration. For Gym locomotion tasks, \ABBR either matches or surpasses RL-only fine-tuning. In ~\Cref{fig:dppo_perf}, DPPO degrades after peaking on \texttt{Hopper}, while \ABBR avoids this drop and ultimately surpasses it by $16\%$ in rewards.

\begin{figure}[H]
    \centering
    \includegraphics[width=\linewidth]{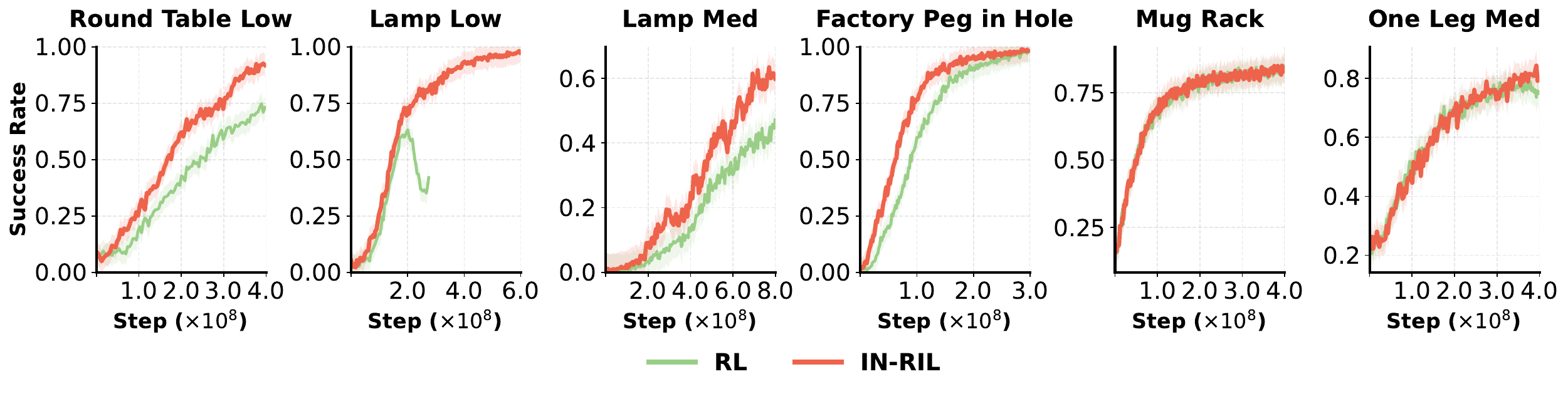}
    \caption{Comparing \ABBR and with fine-tuning on FurnitureBench using residual PPO.}
    \label{fig:residual_furniture_perf}
\end{figure}

FurnitureBench features multi-stage furniture assembly with sparse rewards—conditions that are particularly difficult for RL agents, especially when IL pre-training converges at low success rates. As shown in Table~\ref{tab:pretrain_results_furniture}, pre-training success rates for 3 tasks remain below $5\%$, with only \texttt{One-Leg Low} exceeding $30\%$. Meanwhile, \ABBR significantly outperforms residual PPO across most tasks, as shown in \Cref{tab:furniture_residual}, when consuming the same amount of environment steps. For the challenging \texttt{Lamp Low} task, RL-only fine-tuning frequently collapsed during training, while \ABBR maintains stable learning dynamics across multiple runs. On \texttt{Round-Table Low}, where pre-training achieves only $5\%$ success rate, \ABBR reaches $73\%$ success rate with approximately $\times 10^8$ fewer environment interactions than RL-only fine-tuning with $25\%$ improvement in sample efficiency.

\begin{table}[h!]
\centering
\small
\setlength{\tabcolsep}{4pt}
\begin{tabular}{l|cc|cc}
\toprule
\textbf{Task} & \textbf{IN-RIL (Residual PPO)} & \textbf{Residual PPO} & \textbf{DPPO} & \textbf{IDQL} \\
\midrule
Lamp low              & \textbf{0.98} & 0.63 & 0.85 & 0.11 \\
Lamp med              & \textbf{0.67} & 0.46 & 0.36 & 0.01 \\
Round table low       & \textbf{0.93} & 0.73 & 0.88 & 0.09 \\
One leg low           & 0.94 & \textbf{0.95} & 0.92 & 0.45 \\
One leg med           & \textbf{0.82} & 0.74 & 0.80 & 0.24 \\
% Factory peg in hole   & \textbf{0.93} & 0.92 & --   & 0.01 \\
% Mug rack              & \textbf{0.85} & \textbf{0.85} & --   & 0.16 \\
\bottomrule
\end{tabular}
\caption{Comparing \ABBR with other RL fine-tuning algorithms on FurnitureBench. Bold values indicate the best of all. For each method and task, we report the best success rates among all the checkpoints.}
\label{tab:furniture_residual}
\end{table}

\subsection{Ablation Studies on Interleaving Ratio $m$}\label{sec:ablation}

\begin{figure*}[htb]
\centering
    \centering
    \includegraphics[width=\linewidth]{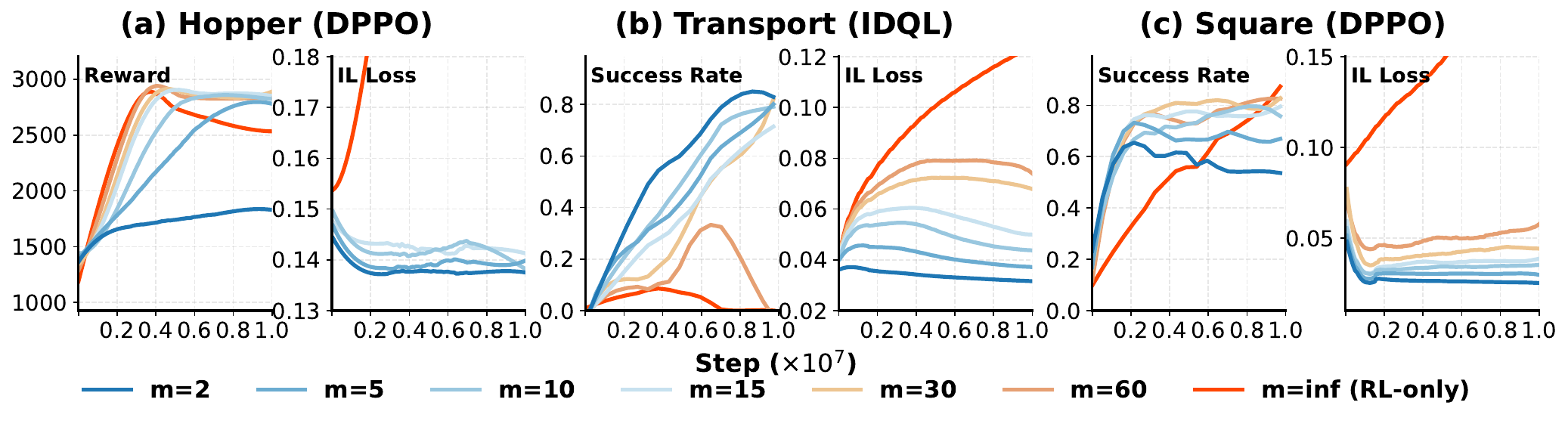}
    \caption{The impact of the interleaving period $m$ on \ABBR RL performance (rewards), and IL performance (IL losses). We use 7 different values for $m$, and train the agent with all the values using $10^7$ environment steps. The figure shows how RL rewards and IL losses change with different $m$. The curves are smoothed using a Savitzky-Golay filter to better show the patterns.}
    \label{fig:ablation_m}
\end{figure*}

Next, we investigate how the interleaving period $m$ affects the learning dynamics of \ABBR by examining changes in both online performance metrics (RL rewards) and offline performance metrics (IL losses) under different values of $m$. For RL-only fine-tuning ($m=\infty$), we compute IL losses to monitor how well the policy maintains fidelity to demonstrations during fine-tuning, but without updating the policy based on these losses. We evaluate \ABBR with seven different values of $m$ on Gym \texttt{Hopper}, Robomimic \texttt{Transport}, and Robomimic \texttt{Square}. In particular, Figure~\ref{fig:ablation_m} reveals several key insights about \ABBR's behavior:

\textbf{Double Descent of IL Losses.} For RL-only fine-tuning ($m=\infty$), IL losses increase dramatically as RL exploration drives the policy away from the pre-trained behavior. In contrast, \ABBR maintains controlled IL loss trajectories. Most remarkably, we observe that IL losses often experience a "double descent" phenomenon—after initially increasing, they begin decreasing again despite the pre-trained policy having fully converged. This empirically validates our hypothesis illustrated in Figure~\ref{fig:interleaving_mechanism} that RL exploration can help IL escape local minima, enabling discovery of superior demonstration-aligned policies that would be inaccessible through IL alone.

\textbf{Enhanced Sample Efficiency.} Figure~\ref{fig:ablation_m}(c) demonstrates that \ABBR dramatically improves the sample efficiency of DPPO, particularly during early fine-tuning. \ABBR converges to high success rates within just $0.4\times 10^7$ steps, while DPPO alone requires approximately $0.9\times 10^7$ steps (2.25× more environment interactions) to achieve comparable performance.

\textbf{Improved Stability.} As shown in Figure~\ref{fig:ablation_m}(a), overly aggressive exploration in RL-only approaches can degrade performance after $0.4\times 10^7$ steps. \ABBR prevents this degradation across multiple interleaving ratios by maintaining IL losses within an appropriate range, effectively constraining exploration to promising regions of the policy space.

\textbf{Guided Exploration.} Figure~\ref{fig:ablation_m}(b) illustrates a critical advantage of \ABBR: on challenging tasks where IDQL fine-tuning alone fails due to ungrounded exploration, \ABBR successfully guides the agent toward task completion. By periodically refreshing the agent's memory of expert demonstrations through IL gradients, \ABBR effectively structures exploration, enabling success on tasks that RL-only approaches cannot solve.

\begin{wrapfigure}{r}{0.40\textwidth}\vspace{-0.5in}
    \centering
    \includegraphics[width=\linewidth]{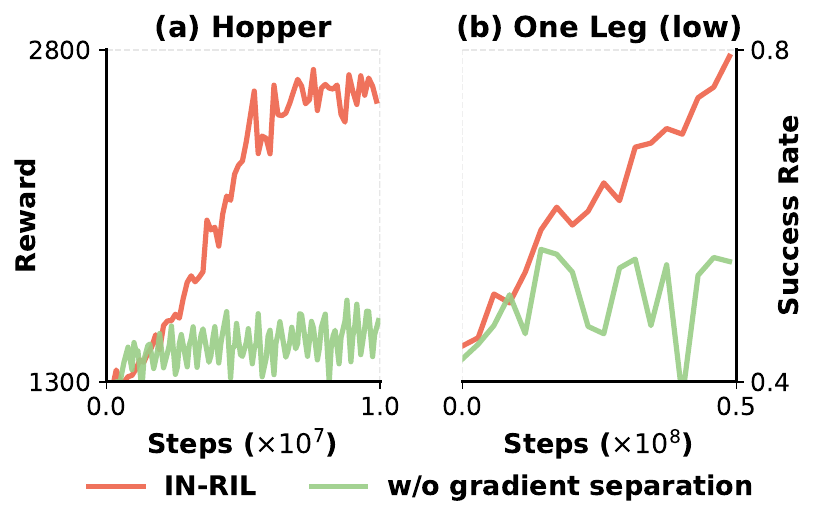}
    \caption{Impact of separation of gradients on \texttt{Hopper} using DPPO and \texttt{One-Leg} (Low) using residual PPO.}
    \label{fig:ablation_gradient}
\end{wrapfigure}

\subsection{Ablation of Separation of RL and IL Gradients.}\label{sec:ablation_gradient_separation} 
When simultaneously leveraging IL and RL gradients to update policy networks, resolving potential interference between these distinct optimization objectives is crucial. When implementing gradient separation for \ABBR with network separation, IL and RL gradients are naturally separated. In contrast, full-network fine-tuning, applies both gradients to the same network. To mitigate interference, we compute IL and RL gradients and apply gradient surgery before performing one gradient step to update the network. \Cref{fig:ablation_gradient} demonstrates that naive interleaving of IL and RL objectives without proper gradient management can significantly impair policy performance, while both separation strategies enable successful task completion.

\section{Conclusion}

We presented \ABBR, a principled approach that enhances robotic policy learning by strategically interleaving IL and RL updates. Our framework maintains stability while enabling exploration through periodic IL regularization, coupled with a gradient separation mechanism that effectively combines complementary learning signals. Theoretical analysis establishes convergence guarantees and sample efficiency conditions, which align with our empirical validation across 14 diverse robot tasks from three benchmarks, demonstrating up to 6.25× improvement in success rate over standard RL finetuning. \ABBR functions as a versatile plugin compatible with various state-of-the-art RL algorithms, substantially enhancing performance on both long- and short-horizon tasks with sparse or dense rewards. Future work will explore adaptive mechanisms to dynamically adjust the interleaving ratio based on gradient alignment during training, extend our approach to domains beyond robotics, and investigate additional techniques to further enhance the synergy between IL and RL objectives.

\bibliographystyle{unsrt}  
\bibliography{references}  %%% Remove comment to use the external .bib file (using bibtex).

\newpage
\appendix

{\bf \Large Appendix}
\section{Justifications on the Assumptions}

\begin{assumption}[Pretraining Performance]
The initial policy parameters $\theta_0$ obtained from pretraining satisfies $ \mathcal{L}_{\mathrm{IL}}(\theta_0) - \mathcal{L}_{\mathrm{IL}}(\theta^*) \leq \epsilon_{\mathrm{IL}}$, where $\epsilon_{\mathrm{IL}} > 0$ is a constant and $\theta^*$ is the optimal solution for optimizing the IL objective. \label{asu:pre}
\end{assumption}

\begin{assumption}[Data Coverage]
The expert demonstration dataset $\mathcal{D}_{\text{exp}}$ provides sufficient coverage of the state space relevant for the target task. Specifically, there exists a constant $C_{\text{coverage}} > 0$ such that:
\begin{align*}
    \mathbb{E}_{s \sim \mu^*}[\min_{s' \in \mathcal{D}_{\text{exp}}} \|s - s'\|] \leq C_{\text{coverage}}
\end{align*}
where $\mu^*$ is the state distribution of the optimal policy for the target task. \label{asu:cov}
\end{assumption}

\begin{assumption}[Smoothness of Objectives]
Both the IL and RL objectives are $L$-smooth:
\begin{align*}
    \|\nabla_\theta\mathcal{L}_{\mathrm{IL}}(\theta) - \nabla_\theta\mathcal{L}_{\mathrm{IL}}(\theta')\| &\leq L_{\mathrm{IL}}\|\theta - \theta'\|, \quad \forall\theta, \theta'\\
    \|\nabla_\theta\mathcal{L}_{\mathrm{RL}}(\theta) - \nabla_\theta\mathcal{L}_{\mathrm{RL}}(\theta')\| &\leq L_{\mathrm{RL}}\|\theta - \theta'\|, \quad \forall\theta, \theta'
\end{align*} \label{asu:smooth}
\end{assumption}

\begin{assumption}[Bounded Variance]
The stochastic gradients have bounded variance:
\begin{align*}
    \mathbb{E}[\|\nabla_\theta\mathcal{L}_{\mathrm{IL}}(\theta) - \widehat{\nabla}_\theta\mathcal{L}_{\mathrm{IL}}(\theta)\|^2] &\leq \frac{\sigma^2_{\mathrm{IL}}}{N_{\mathrm{IL}}}\\
    \mathbb{E}[\|\nabla_\theta\mathcal{L}_{\mathrm{RL}}(\theta) - \widehat{\nabla}_\theta\mathcal{L}_{\mathrm{RL}}(\theta)\|^2] &\leq \frac{\sigma^2_{\mathrm{RL}}}{N_{\mathrm{RL}}}
\end{align*}
where $\widehat{\nabla}$ represents the stochastic gradient estimate, and $N_{\mathrm{IL}}$ and $N_{\mathrm{RL}}$ are the batch sizes. \label{asu:variance}
\end{assumption}

We first provide the detailed justification on the assumptions used in Section 2.

\paragraph{Assumption 1 (Near-Optimal IL Performance)} This assumption reflects the practical setting where we start from a pre-trained policy that already performs well on demonstration data. It's commonly used in transfer learning and foundation model literature where models are first trained on large datasets before task-specific adaptation \cite{brown2020language,bommasani2021opportunities}. The small constant $\epsilon_{IL}$ quantifies how close the initial policy is to optimal imitation performance, capturing the idea that while the model has learned a good behavioral prior, there's still room for improvement through reinforcement learning.

\paragraph{Assumption 2 (Data Coverage)}  The data coverage assumption ensures that the expert demonstrations provide adequate representation of the states relevant to the target task. This is a standard assumption in imitation learning \cite{ross2011reduction,daume2009search} and reflects the intuition that learning can only occur for regions of the state space that have been demonstrated. The constant $C_{coverage}$ quantifies the maximum expected distance between a state from the optimal policy and its nearest neighbor in the demonstration dataset, with smaller values indicating better coverage.

\paragraph{Assumption 3 (Smoothness of Objectives)} Smoothness is a standard assumption in optimization theory \cite{nesterov2004introductory,bottou2018optimization}that ensures the gradient doesn't change too drastically between nearby points. This enables reliable gradient-based optimization and allows us to derive convergence rates. Practically, this assumption holds for most neural network architectures with commonly used activation functions when properly normalized, and is critical for establishing the descent lemma used in our analysis.

\paragraph{Assumption 4 (Gradient Alignment)} 
This assumption characterizes the geometric relationship between the gradients of the IL and RL objectives. The parameter $\rho(t)$ captures the cosine similarity between these gradients, with positive values indicating opposing gradients and negative values indicating aligned gradients. Similar assumptions appear in multi-task learning literature \cite{sener2018multi} and multi-objective optimization \cite{desideri2012multiple}. This formulation allows us to analyze how the IL updates affect progress on the RL objective, which is crucial for determining the optimal interleaving strategy.

\paragraph{Assumption 5 (Bounded Variance)}
The bounded variance assumption is standard in stochastic optimization literature \cite{robbins1951stochastic,bottou2018optimization} and reflects the fact that stochastic gradient estimates contain noise due to mini-batch sampling. The variance terms $\sigma^2_{IL}$ and $\sigma^2_{RL}$ quantify this noise, with the variance decreasing as batch size increases. This assumption is necessary for establishing convergence rates in the presence of stochastic gradients and is satisfied in practice when using proper mini-batch sampling techniques.

Based on these assumptions, we first establish the following key results (proofs in the appendix). We begin our theoretical analysis by establishing convergence analysis for RL-only finetune and \ABBR, respectively.

\begin{theorem}[Convergence of RL-Only Training]
Under Assumptions \ref{asu:pre}-\ref{asu:variance}, with learning rate $\alpha_{\mathrm{RL}} = \frac{c_{\mathrm{RL}}}{L_{\mathrm{RL}}}$ for $c_{\mathrm{RL}} \in (0, 1)$, RL-only training for $T$ iterations achieves:
\begin{align*}
    \min_{0 \leq t < T}\mathbb{E}[\|\nabla\mathcal{L}_{\mathrm{RL}}(\theta_t)\|^2] \leq \frac{2L_{\mathrm{RL}}(\mathcal{L}_{\mathrm{RL}}(\theta_0) - \mathcal{L}^*_{\mathrm{RL}})}{c_{\mathrm{RL}}(1 - \frac{c_{\mathrm{RL}}}{2})T} + \frac{c_{\mathrm{RL}}\sigma^2_{\mathrm{RL}}}{(1 - \frac{c_{\mathrm{RL}}}{2})N_{\mathrm{RL}}}
\end{align*}
\label{thm:covRL}
\end{theorem}

\begin{theorem}[Convergence with \ABBR]
Under Assumptions \ref{asu:pre}-\ref{asu:variance}, with learning rates $\alpha_{\mathrm{IL}} = \frac{c_{\mathrm{IL}}}{L_{\mathrm{IL}}}$ and $\alpha_{\mathrm{RL}} = \frac{c_{\mathrm{RL}}}{L_{\mathrm{RL}}}$ for $c_{\mathrm{IL}}, c_{\mathrm{RL}} \in (0, 1)$, interleaved 1:$m(t)$ training for $T$ cycles achieves:
\begin{align*}
    \min_{0 \leq t < T}\mathbb{E}[\|\nabla\mathcal{L}_{\mathrm{RL}}(\theta_t)\|^2] \leq \frac{2(L_{\mathrm{RL}}(\mathcal{L}_{\mathrm{RL}}(\theta_0) - \mathcal{L}^*_{\mathrm{RL}}) - \Delta_{\mathrm{IL-RL}})}{c_{\mathrm{RL}}(1 - \frac{c_{\mathrm{RL}}}{2})\bar{m}T} + \frac{c_{\mathrm{RL}}\sigma^2_{\mathrm{RL}}}{(1 - \frac{c_{\mathrm{RL}}}{2})N_{\mathrm{RL}}}
\end{align*}

where $\bar{m} = \frac{1}{T}\sum_{t=0}^{T-1}m(t)$ is the average interleaving ratio, and $\Delta_{\mathrm{IL-RL}}$ represents the benefit from IL regularization, i.e., $\Delta_{\mathrm{IL-RL}} =- \sum_{t=0}^{T-1}\frac{c_{\mathrm{IL}}\rho(t)}{L_{\mathrm{IL}}}\|\nabla\mathcal{L}_{\mathrm{IL}}(\theta_t)\| \cdot \|\nabla\mathcal{L}_{\mathrm{RL}}(\theta_t)\| - \frac{c^2_{\mathrm{IL}}\sigma^2_{\mathrm{IL}}T}{2L_{\mathrm{IL}}N_{\mathrm{IL}}}$ 
\label{thm:covINT}
\end{theorem}

\Cref{thm:covRL} establishes that with appropriate learning rates, RL-only finetuning achieves the standard $O(1/T)$ convergence rate for smooth objectives. \Cref{thm:covINT} reveals that \ABBR can achieve better convergence guarantees than RL-only finetuning through the regularization benefit term $\Delta_{\mathrm{IL-RL}}$. This term captures how IL updates can enhance RL performance, especially when gradient alignment is favorable ($\rho(t) < 0$). Having established the benefits of \ABBR, we now derive the optimal ratio of RL updates to IL updates. This ratio is crucial for balancing the stability provided by IL updates with the performance improvements offered by RL updates.

\section{Proof of \Cref{thm:covRL}}

We first establish the following technical lemmas that will be used in the proof of the main theorems.

\begin{lemma}[Descent Lemma]
For a function $f$ with $L$-smoothness, we have:
\begin{align*}
    f(y) \leq f(x) + \langle\nabla f(x), y-x\rangle + \frac{L}{2}\|y-x\|^2
\end{align*}
\end{lemma}

\begin{lemma}[Progress Bound for Gradient Descent]
For a function $f$ with $L$-smoothness and step size $\alpha = \frac{c}{L}$ where $c \in (0,1)$, one step of gradient descent gives:
\begin{align*}
    f(x - \alpha\nabla f(x)) \leq f(x) - \frac{c(1-\frac{c}{2})}{L}\|\nabla f(x)\|^2
\end{align*}
\end{lemma}

\begin{lemma}[Error Bound for Stochastic Gradient Descent]
For a function $f$ with $L$-smoothness, step size $\alpha = \frac{c}{L}$ where $c \in (0,1)$, and stochastic gradient $\widehat{\nabla} f(x)$ with bounded variance $\mathbb{E}[\|\nabla f(x) - \widehat{\nabla} f(x)\|^2] \leq \frac{\sigma^2}{N}$, one step of stochastic gradient descent gives:
\begin{align*}
    \mathbb{E}[f(x - \alpha\widehat{\nabla} f(x))] \leq f(x) - \frac{c(1-\frac{c}{2})}{L}\|\nabla f(x)\|^2 + \frac{c^2\sigma^2}{2LN}
\end{align*}
\end{lemma}

\begin{proof}
The RL-only update rule is:
\begin{align*}
    \theta_{t+1} = \theta_t - \alpha_{\mathrm{RL}}\widehat{\nabla}_\theta\mathcal{L}_{\mathrm{RL}}(\theta_t)
\end{align*}

Where $\widehat{\nabla}_\theta\mathcal{L}_{\mathrm{RL}}(\theta_t)$ is the stochastic gradient estimate. Applying Lemma 3 to the RL objective, with $\alpha_{\mathrm{RL}} = \frac{c_{\mathrm{RL}}}{L_{\mathrm{RL}}}$:
\begin{align*}
    \mathbb{E}[\mathcal{L}_{\mathrm{RL}}(\theta_{t+1})] &\leq \mathcal{L}_{\mathrm{RL}}(\theta_t) - \frac{c_{\mathrm{RL}}(1-\frac{c_{\mathrm{RL}}}{2})}{L_{\mathrm{RL}}}\|\nabla\mathcal{L}_{\mathrm{RL}}(\theta_t)\|^2 + \frac{c_{\mathrm{RL}}^2\sigma^2_{\mathrm{RL}}}{2L_{\mathrm{RL}}N_{\mathrm{RL}}}
\end{align*}

Rearranging:
\begin{align*}
    \frac{c_{\mathrm{RL}}(1-\frac{c_{\mathrm{RL}}}{2})}{L_{\mathrm{RL}}}\|\nabla\mathcal{L}_{\mathrm{RL}}(\theta_t)\|^2 &\leq \mathcal{L}_{\mathrm{RL}}(\theta_t) - \mathbb{E}[\mathcal{L}_{\mathrm{RL}}(\theta_{t+1})] + \frac{c_{\mathrm{RL}}^2\sigma^2_{\mathrm{RL}}}{2L_{\mathrm{RL}}N_{\mathrm{RL}}}
\end{align*}

Summing from $t = 0$ to $T-1$:
\begin{align*}
    \frac{c_{\mathrm{RL}}(1-\frac{c_{\mathrm{RL}}}{2})}{L_{\mathrm{RL}}}\sum_{t=0}^{T-1}\|\nabla\mathcal{L}_{\mathrm{RL}}(\theta_t)\|^2 &\leq \mathcal{L}_{\mathrm{RL}}(\theta_0) - \mathbb{E}[\mathcal{L}_{\mathrm{RL}}(\theta_{T})] + \frac{c_{\mathrm{RL}}^2\sigma^2_{\mathrm{RL}}T}{2L_{\mathrm{RL}}N_{\mathrm{RL}}}
\end{align*}

By Assumption 6, $\mathcal{L}_{\mathrm{RL}}(\theta_T) \geq \mathcal{L}^*_{\mathrm{RL}}$ (the optimal value), so:
\begin{align*}
    \frac{c_{\mathrm{RL}}(1-\frac{c_{\mathrm{RL}}}{2})}{L_{\mathrm{RL}}}\sum_{t=0}^{T-1}\|\nabla\mathcal{L}_{\mathrm{RL}}(\theta_t)\|^2 &\leq \mathcal{L}_{\mathrm{RL}}(\theta_0) - \mathcal{L}^*_{\mathrm{RL}} + \frac{c_{\mathrm{RL}}^2\sigma^2_{\mathrm{RL}}T}{2L_{\mathrm{RL}}N_{\mathrm{RL}}}
\end{align*}

By the pigeonhole principle, there must exist at least one iteration $t^* \in \{0, 1, \ldots, T-1\}$ such that:
\begin{align*}
    \|\nabla\mathcal{L}_{\mathrm{RL}}(\theta_{t^*})\|^2 \leq \frac{1}{T}\sum_{t=0}^{T-1}\|\nabla\mathcal{L}_{\mathrm{RL}}(\theta_t)\|^2
\end{align*}

Therefore:
\begin{align*}
    \min_{0 \leq t < T}\|\nabla\mathcal{L}_{\mathrm{RL}}(\theta_t)\|^2 \leq \frac{1}{T}\sum_{t=0}^{T-1}\|\nabla\mathcal{L}_{\mathrm{RL}}(\theta_t)\|^2 \leq \frac{L_{\mathrm{RL}}(\mathcal{L}_{\mathrm{RL}}(\theta_0) - \mathcal{L}^*_{\mathrm{RL}})}{c_{\mathrm{RL}}(1-\frac{c_{\mathrm{RL}}}{2})T} + \frac{c_{\mathrm{RL}}^2\sigma^2_{\mathrm{RL}}}{2c_{\mathrm{RL}}(1-\frac{c_{\mathrm{RL}}}{2})N_{\mathrm{RL}}}
\end{align*}

Simplifying the second term:
\begin{align*}
    \min_{0 \leq t < T}\|\nabla\mathcal{L}_{\mathrm{RL}}(\theta_t)\|^2 &\leq \frac{L_{\mathrm{RL}}(\mathcal{L}_{\mathrm{RL}}(\theta_0) - \mathcal{L}^*_{\mathrm{RL}})}{c_{\mathrm{RL}}(1-\frac{c_{\mathrm{RL}}}{2})T} + \frac{c_{\mathrm{RL}}\sigma^2_{\mathrm{RL}}}{2(1-\frac{c_{\mathrm{RL}}}{2})N_{\mathrm{RL}}}
\end{align*}

Taking expectation and adjusting the constant in the second term:
\begin{align*}
    \min_{0 \leq t < T}\mathbb{E}[\|\nabla\mathcal{L}_{\mathrm{RL}}(\theta_t)\|^2] \leq \frac{2L_{\mathrm{RL}}(\mathcal{L}_{\mathrm{RL}}(\theta_0) - \mathcal{L}^*_{\mathrm{RL}})}{c_{\mathrm{RL}}(1 - \frac{c_{\mathrm{RL}}}{2})T} + \frac{c_{\mathrm{RL}}\sigma^2_{\mathrm{RL}}}{(1 - \frac{c_{\mathrm{RL}}}{2})N_{\mathrm{RL}}}
\end{align*}

For the IL performance bound, we use the $L_{\mathrm{IL}}$-smoothness of the IL objective (Assumption 3):
\begin{align*}
    \mathcal{L}_{\mathrm{IL}}(\theta_T) - \mathcal{L}_{\mathrm{IL}}(\theta_0) &\leq \langle\nabla\mathcal{L}_{\mathrm{IL}}(\theta_0), \theta_T - \theta_0\rangle + \frac{L_{\mathrm{IL}}}{2}\|\theta_T - \theta_0\|^2 \\
    &\leq \|\nabla\mathcal{L}_{\mathrm{IL}}(\theta_0)\| \cdot \|\theta_T - \theta_0\| + \frac{L_{\mathrm{IL}}}{2}\|\theta_T - \theta_0\|^2
\end{align*}

From Assumption 1 (Near-Optimal IL Performance), the gradient $\|\nabla\mathcal{L}_{\mathrm{IL}}(\theta_0)\|$ is small. For simplicity, we can absorb this term into the quadratic term:
\begin{align*}
    \mathcal{L}_{\mathrm{IL}}(\theta_T) - \mathcal{L}_{\mathrm{IL}}(\theta_0) \leq \frac{L_{\mathrm{IL}}}{2}\|\theta_T - \theta_0\|^2
\end{align*}

Combining with Assumption 1, we have:
\begin{align*}
    \mathcal{L}_{\mathrm{IL}}(\theta_T) - \mathcal{L}_{\mathrm{IL}}(\theta^*) &= \mathcal{L}_{\mathrm{IL}}(\theta_T) - \mathcal{L}_{\mathrm{IL}}(\theta_0) + \mathcal{L}_{\mathrm{IL}}(\theta_0) - \mathcal{L}_{\mathrm{IL}}(\theta^*) \\
    &\leq \frac{L_{\mathrm{IL}}}{2}\|\theta_T - \theta_0\|^2 + \epsilon_{\mathrm{IL}}
\end{align*}
This completes the proof.
\end{proof}

\section{Proof of \Cref{thm:covINT}}

\begin{proof}
The interleaved training consists of cycles where each cycle has one IL update followed by $m(t)$ RL updates. Let $\theta_t$ denote the parameters at the beginning of cycle $t$, and $\theta_{t+\frac{j}{1+m(t)}}$ denote the parameters after the $j$-th update within cycle $t$.

First, let's analyze the IL update within cycle $t$:
\begin{align*}
    \theta_{t+\frac{1}{1+m(t)}} = \theta_t - \alpha_{\mathrm{IL}}\widehat{\nabla}\mathcal{L}_{\mathrm{IL}}(\theta_t)
\end{align*}

Applying Lemma 3 to the IL objective with $\alpha_{\mathrm{IL}} = \frac{c_{\mathrm{IL}}}{L_{\mathrm{IL}}}$:
\begin{align*}
    \mathbb{E}[\mathcal{L}_{\mathrm{IL}}(\theta_{t+\frac{1}{1+m(t)}})] &\leq \mathcal{L}_{\mathrm{IL}}(\theta_t) - \frac{c_{\mathrm{IL}}(1-\frac{c_{\mathrm{IL}}}{2})}{L_{\mathrm{IL}}}\|\nabla\mathcal{L}_{\mathrm{IL}}(\theta_t)\|^2 + \frac{c_{\mathrm{IL}}^2\sigma^2_{\mathrm{IL}}}{2L_{\mathrm{IL}}N_{\mathrm{IL}}}
\end{align*}

Now, let's analyze how this IL update affects the RL objective. Using the smoothness of the RL objective (Assumption 3):
\begin{align*}
    \mathcal{L}_{\mathrm{RL}}(\theta_{t+\frac{1}{1+m(t)}}) &\leq \mathcal{L}_{\mathrm{RL}}(\theta_t) + \langle\nabla\mathcal{L}_{\mathrm{RL}}(\theta_t), \theta_{t+\frac{1}{1+m(t)}} - \theta_t\rangle + \frac{L_{\mathrm{RL}}}{2}\|\theta_{t+\frac{1}{1+m(t)}} - \theta_t\|^2 \\
    &= \mathcal{L}_{\mathrm{RL}}(\theta_t) + \langle\nabla\mathcal{L}_{\mathrm{RL}}(\theta_t), -\alpha_{\mathrm{IL}}\widehat{\nabla}\mathcal{L}_{\mathrm{IL}}(\theta_t)\rangle + \frac{L_{\mathrm{RL}}\alpha_{\mathrm{IL}}^2}{2}\|\widehat{\nabla}\mathcal{L}_{\mathrm{IL}}(\theta_t)\|^2
\end{align*}

Taking expectations and using the fact that $\mathbb{E}[\widehat{\nabla}\mathcal{L}_{\mathrm{IL}}(\theta_t)] = \nabla\mathcal{L}_{\mathrm{IL}}(\theta_t)$ (unbiased estimator):
\begin{align*}
    \mathbb{E}[\mathcal{L}_{\mathrm{RL}}(\theta_{t+\frac{1}{1+m(t)}})] &\leq \mathcal{L}_{\mathrm{RL}}(\theta_t) - \alpha_{\mathrm{IL}}\langle\nabla\mathcal{L}_{\mathrm{RL}}(\theta_t), \nabla\mathcal{L}_{\mathrm{IL}}(\theta_t)\rangle + \frac{L_{\mathrm{RL}}\alpha_{\mathrm{IL}}^2}{2}\mathbb{E}[\|\widehat{\nabla}\mathcal{L}_{\mathrm{IL}}(\theta_t)\|^2]
\end{align*}

Using Assumption 4 (Gradient align*ment):
\begin{align*}
    \langle\nabla\mathcal{L}_{\mathrm{IL}}(\theta_t), \nabla\mathcal{L}_{\mathrm{RL}}(\theta_t)\rangle = -\rho(t)\|\nabla\mathcal{L}_{\mathrm{IL}}(\theta_t)\| \cdot \|\nabla\mathcal{L}_{\mathrm{RL}}(\theta_t)\|
\end{align*}

And using Assumption 5 (Bounded Variance):
\begin{align*}
    \mathbb{E}[\|\widehat{\nabla}\mathcal{L}_{\mathrm{IL}}(\theta_t)\|^2] \leq \|\nabla\mathcal{L}_{\mathrm{IL}}(\theta_t)\|^2 + \frac{\sigma^2_{\mathrm{IL}}}{N_{\mathrm{IL}}}
\end{align*}

We get:
\begin{align*}
    \mathbb{E}[\mathcal{L}_{\mathrm{RL}}(\theta_{t+\frac{1}{1+m(t)}})] &\leq \mathcal{L}_{\mathrm{RL}}(\theta_t) + \alpha_{\mathrm{IL}}\rho(t)\|\nabla\mathcal{L}_{\mathrm{IL}}(\theta_t)\| \cdot \|\nabla\mathcal{L}_{\mathrm{RL}}(\theta_t)\| \\
    &\quad + \frac{L_{\mathrm{RL}}\alpha_{\mathrm{IL}}^2}{2}\left(\|\nabla\mathcal{L}_{\mathrm{IL}}(\theta_t)\|^2 + \frac{\sigma^2_{\mathrm{IL}}}{N_{\mathrm{IL}}}\right)
\end{align*}

Substituting $\alpha_{\mathrm{IL}} = \frac{c_{\mathrm{IL}}}{L_{\mathrm{IL}}}$:
\begin{align*}
    \mathbb{E}[\mathcal{L}_{\mathrm{RL}}(\theta_{t+\frac{1}{1+m(t)}})] &\leq \mathcal{L}_{\mathrm{RL}}(\theta_t) + \frac{c_{\mathrm{IL}}}{L_{\mathrm{IL}}}\rho(t)\|\nabla\mathcal{L}_{\mathrm{IL}}(\theta_t)\| \cdot \|\nabla\mathcal{L}_{\mathrm{RL}}(\theta_t)\| \\
    &\quad + \frac{L_{\mathrm{RL}}c_{\mathrm{IL}}^2}{2L_{\mathrm{IL}}^2}\left(\|\nabla\mathcal{L}_{\mathrm{IL}}(\theta_t)\|^2 + \frac{\sigma^2_{\mathrm{IL}}}{N_{\mathrm{IL}}}\right)
\end{align*}

Now, let's analyze the $m(t)$ RL updates. For each RL update $j \in \{1, \ldots, m(t)\}$:
\begin{align*}
    \theta_{t+\frac{1+j}{1+m(t)}} = \theta_{t+\frac{j}{1+m(t)}} - \alpha_{\mathrm{RL}}\widehat{\nabla}\mathcal{L}_{\mathrm{RL}}(\theta_{t+\frac{j}{1+m(t)}})
\end{align*}

Applying Lemma 3 to each RL update, with $\alpha_{\mathrm{RL}} = \frac{c_{\mathrm{RL}}}{L_{\mathrm{RL}}}$:
\begin{align*}
    \mathbb{E}[\mathcal{L}_{\mathrm{RL}}(\theta_{t+\frac{1+j}{1+m(t)}})] &\leq \mathcal{L}_{\mathrm{RL}}(\theta_{t+\frac{j}{1+m(t)}}) - \frac{c_{\mathrm{RL}}(1-\frac{c_{\mathrm{RL}}}{2})}{L_{\mathrm{RL}}}\|\nabla\mathcal{L}_{\mathrm{RL}}(\theta_{t+\frac{j}{1+m(t)}})\|^2 \\
    &\quad + \frac{c_{\mathrm{RL}}^2\sigma^2_{\mathrm{RL}}}{2L_{\mathrm{RL}}N_{\mathrm{RL}}}
\end{align*}

For simplicity of analysis, we can bound the gradient norms at intermediate steps using the gradient at the beginning of the cycle:
\begin{align*}
    \|\nabla\mathcal{L}_{\mathrm{RL}}(\theta_{t+\frac{j}{1+m(t)}})\|^2 \geq (1-\delta)^2\|\nabla\mathcal{L}_{\mathrm{RL}}(\theta_t)\|^2
\end{align*}
for some small $\delta > 0$ that depends on the learning rates and smoothness constants. This approximation is reasonable because the parameters don't change drastically within a cycle when using small learning rates.

With this approximation, we get:
\begin{align*}
    \mathbb{E}[\mathcal{L}_{\mathrm{RL}}(\theta_{t+\frac{1+j}{1+m(t)}})] &\leq \mathcal{L}_{\mathrm{RL}}(\theta_{t+\frac{j}{1+m(t)}}) - \frac{c_{\mathrm{RL}}(1-\frac{c_{\mathrm{RL}}}{2})(1-\delta)^2}{L_{\mathrm{RL}}}\|\nabla\mathcal{L}_{\mathrm{RL}}(\theta_t)\|^2 \\
    &\quad + \frac{c_{\mathrm{RL}}^2\sigma^2_{\mathrm{RL}}}{2L_{\mathrm{RL}}N_{\mathrm{RL}}}
\end{align*}

Applying this recursively for all $m(t)$ RL updates and combining with the effect of the IL update, we get:
\begin{align*}
    \mathbb{E}[\mathcal{L}_{\mathrm{RL}}(\theta_{t+1})] &\leq \mathcal{L}_{\mathrm{RL}}(\theta_t) + \frac{c_{\mathrm{IL}}}{L_{\mathrm{IL}}}\rho(t)\|\nabla\mathcal{L}_{\mathrm{IL}}(\theta_t)\| \cdot \|\nabla\mathcal{L}_{\mathrm{RL}}(\theta_t)\| \\
    &\quad + \frac{L_{\mathrm{RL}}c_{\mathrm{IL}}^2}{2L_{\mathrm{IL}}^2}\left(\|\nabla\mathcal{L}_{\mathrm{IL}}(\theta_t)\|^2 + \frac{\sigma^2_{\mathrm{IL}}}{N_{\mathrm{IL}}}\right) \\
    &\quad - m(t)\frac{c_{\mathrm{RL}}(1-\frac{c_{\mathrm{RL}}}{2})(1-\delta)^2}{L_{\mathrm{RL}}}\|\nabla\mathcal{L}_{\mathrm{RL}}(\theta_t)\|^2 + m(t)\frac{c_{\mathrm{RL}}^2\sigma^2_{\mathrm{RL}}}{2L_{\mathrm{RL}}N_{\mathrm{RL}}}
\end{align*}

For simplicity, we'll absorb $(1-\delta)^2$ into the constants. Rearranging:
\begin{align*}
    m(t)\frac{c_{\mathrm{RL}}(1-\frac{c_{\mathrm{RL}}}{2})}{L_{\mathrm{RL}}}\|\nabla\mathcal{L}_{\mathrm{RL}}(\theta_t)\|^2 &\leq \mathcal{L}_{\mathrm{RL}}(\theta_t) - \mathbb{E}[\mathcal{L}_{\mathrm{RL}}(\theta_{t+1})] \\
    &\quad + \frac{c_{\mathrm{IL}}}{L_{\mathrm{IL}}}\rho(t)\|\nabla\mathcal{L}_{\mathrm{IL}}(\theta_t)\| \cdot \|\nabla\mathcal{L}_{\mathrm{RL}}(\theta_t)\| \\
    &\quad + \frac{L_{\mathrm{RL}}c_{\mathrm{IL}}^2}{2L_{\mathrm{IL}}^2}\|\nabla\mathcal{L}_{\mathrm{IL}}(\theta_t)\|^2 + \frac{L_{\mathrm{RL}}c_{\mathrm{IL}}^2\sigma^2_{\mathrm{IL}}}{2L_{\mathrm{IL}}^2N_{\mathrm{IL}}} \\
    &\quad + m(t)\frac{c_{\mathrm{RL}}^2\sigma^2_{\mathrm{RL}}}{2L_{\mathrm{RL}}N_{\mathrm{RL}}}
\end{align*}

Summing over $t = 0$ to $T-1$:
\begin{align*}
    \sum_{t=0}^{T-1}m(t)\frac{c_{\mathrm{RL}}(1-\frac{c_{\mathrm{RL}}}{2})}{L_{\mathrm{RL}}}\|\nabla\mathcal{L}_{\mathrm{RL}}(\theta_t)\|^2 &\leq \mathcal{L}_{\mathrm{RL}}(\theta_0) - \mathbb{E}[\mathcal{L}_{\mathrm{RL}}(\theta_{T})] \\
    &\quad + \sum_{t=0}^{T-1}\frac{c_{\mathrm{IL}}}{L_{\mathrm{IL}}}\rho(t)\|\nabla\mathcal{L}_{\mathrm{IL}}(\theta_t)\| \cdot \|\nabla\mathcal{L}_{\mathrm{RL}}(\theta_t)\| \\
    &\quad + \sum_{t=0}^{T-1}\frac{L_{\mathrm{RL}}c_{\mathrm{IL}}^2}{2L_{\mathrm{IL}}^2}\|\nabla\mathcal{L}_{\mathrm{IL}}(\theta_t)\|^2 + T\frac{L_{\mathrm{RL}}c_{\mathrm{IL}}^2\sigma^2_{\mathrm{IL}}}{2L_{\mathrm{IL}}^2N_{\mathrm{IL}}} \\
    &\quad + \sum_{t=0}^{T-1}m(t)\frac{c_{\mathrm{RL}}^2\sigma^2_{\mathrm{RL}}}{2L_{\mathrm{RL}}N_{\mathrm{RL}}}
\end{align*}

By Assumption 6, $\mathcal{L}_{\mathrm{RL}}(\theta_T) \geq \mathcal{L}^*_{\mathrm{RL}}$, so:
\begin{align*}
    \sum_{t=0}^{T-1}m(t)\frac{c_{\mathrm{RL}}(1-\frac{c_{\mathrm{RL}}}{2})}{L_{\mathrm{RL}}}\|\nabla\mathcal{L}_{\mathrm{RL}}(\theta_t)\|^2 &\leq \mathcal{L}_{\mathrm{RL}}(\theta_0) - \mathcal{L}^*_{\mathrm{RL}} \\
    &\quad + \sum_{t=0}^{T-1}\frac{c_{\mathrm{IL}}}{L_{\mathrm{IL}}}\rho(t)\|\nabla\mathcal{L}_{\mathrm{IL}}(\theta_t)\| \cdot \|\nabla\mathcal{L}_{\mathrm{RL}}(\theta_t)\| \\
    &\quad + \sum_{t=0}^{T-1}\frac{L_{\mathrm{RL}}c_{\mathrm{IL}}^2}{2L_{\mathrm{IL}}^2}\|\nabla\mathcal{L}_{\mathrm{IL}}(\theta_t)\|^2 + T\frac{L_{\mathrm{RL}}c_{\mathrm{IL}}^2\sigma^2_{\mathrm{IL}}}{2L_{\mathrm{IL}}^2N_{\mathrm{IL}}} \\
    &\quad + \sum_{t=0}^{T-1}m(t)\frac{c_{\mathrm{RL}}^2\sigma^2_{\mathrm{RL}}}{2L_{\mathrm{RL}}N_{\mathrm{RL}}}
\end{align*}

For the sum of IL gradient norms, we can use the IL update analysis. From our earlier bound on IL updates:
\begin{align*}
    \sum_{t=0}^{T-1}\frac{c_{\mathrm{IL}}(1-\frac{c_{\mathrm{IL}}}{2})}{L_{\mathrm{IL}}}\|\nabla\mathcal{L}_{\mathrm{IL}}(\theta_t)\|^2 \leq \mathcal{L}_{\mathrm{IL}}(\theta_0) - \mathbb{E}[\mathcal{L}_{\mathrm{IL}}(\theta_{T})] + \frac{c_{\mathrm{IL}}^2\sigma^2_{\mathrm{IL}}T}{2L_{\mathrm{IL}}N_{\mathrm{IL}}}
\end{align*}

This gives us:
\begin{align*}
    \sum_{t=0}^{T-1}\|\nabla\mathcal{L}_{\mathrm{IL}}(\theta_t)\|^2 \leq \frac{L_{\mathrm{IL}}(\mathcal{L}_{\mathrm{IL}}(\theta_0) - \mathcal{L}^*_{\mathrm{IL}})}{c_{\mathrm{IL}}(1-\frac{c_{\mathrm{IL}}}{2})} + \frac{c_{\mathrm{IL}}\sigma^2_{\mathrm{IL}}T}{2(1-\frac{c_{\mathrm{IL}}}{2})N_{\mathrm{IL}}}
\end{align*}

Substituting this bound and defining $\bar{m} = \frac{1}{T}\sum_{t=0}^{T-1}m(t)$ as the average interleaving ratio:
\begin{align*}
    \bar{m}T\frac{c_{\mathrm{RL}}(1-\frac{c_{\mathrm{RL}}}{2})}{L_{\mathrm{RL}}}\frac{1}{T}\sum_{t=0}^{T-1}\|\nabla\mathcal{L}_{\mathrm{RL}}(\theta_t)\|^2 &\leq \mathcal{L}_{\mathrm{RL}}(\theta_0) - \mathcal{L}^*_{\mathrm{RL}} \\
    &\quad + \sum_{t=0}^{T-1}\frac{c_{\mathrm{IL}}}{L_{\mathrm{IL}}}\rho(t)\|\nabla\mathcal{L}_{\mathrm{IL}}(\theta_t)\| \cdot \|\nabla\mathcal{L}_{\mathrm{RL}}(\theta_t)\| \\
    &\quad + \frac{L_{\mathrm{RL}}c_{\mathrm{IL}}^2}{2L_{\mathrm{IL}}^2} \cdot \frac{L_{\mathrm{IL}}(\mathcal{L}_{\mathrm{IL}}(\theta_0) - \mathcal{L}^*_{\mathrm{IL}})}{c_{\mathrm{IL}}(1-\frac{c_{\mathrm{IL}}}{2})} + T\frac{L_{\mathrm{RL}}c_{\mathrm{IL}}^2\sigma^2_{\mathrm{IL}}}{2L_{\mathrm{IL}}^2N_{\mathrm{IL}}} \\
    &\quad + \bar{m}T\frac{c_{\mathrm{RL}}^2\sigma^2_{\mathrm{RL}}}{2L_{\mathrm{RL}}N_{\mathrm{RL}}}
\end{align*}

The term with IL gradient norms can be simplified to:
\begin{align*}
    \frac{L_{\mathrm{RL}}c_{\mathrm{IL}}^2}{2L_{\mathrm{IL}}^2} \cdot \frac{L_{\mathrm{IL}}(\mathcal{L}_{\mathrm{IL}}(\theta_0) - \mathcal{L}^*_{\mathrm{IL}})}{c_{\mathrm{IL}}(1-\frac{c_{\mathrm{IL}}}{2})} = \frac{L_{\mathrm{RL}}c_{\mathrm{IL}}}{2L_{\mathrm{IL}}} \cdot \frac{(\mathcal{L}_{\mathrm{IL}}(\theta_0) - \mathcal{L}^*_{\mathrm{IL}})}{(1-\frac{c_{\mathrm{IL}}}{2})}
\end{align*}

By Assumption 1, $\mathcal{L}_{\mathrm{IL}}(\theta_0) - \mathcal{L}^*_{\mathrm{IL}} \leq \epsilon_{\mathrm{IL}}$, which is small. For large enough $T$, this term becomes negligible.

Define the IL regularization benefit:
\begin{align*}
    \Delta_{\mathrm{IL-RL}} = -\sum_{t=0}^{T-1}\frac{c_{\mathrm{IL}}}{L_{\mathrm{IL}}}\rho(t)\|\nabla\mathcal{L}_{\mathrm{IL}}(\theta_t)\| \cdot \|\nabla\mathcal{L}_{\mathrm{RL}}(\theta_t)\| + \frac{c^2_{\mathrm{IL}}\sigma^2_{\mathrm{IL}}T}{2L_{\mathrm{IL}}N_{\mathrm{IL}}}
\end{align*}

With this, our bound becomes:
\begin{align*}
    \bar{m}\frac{c_{\mathrm{RL}}(1-\frac{c_{\mathrm{RL}}}{2})}{L_{\mathrm{RL}}}\frac{1}{T}\sum_{t=0}^{T-1}\|\nabla\mathcal{L}_{\mathrm{RL}}(\theta_t)\|^2 &\leq \frac{\mathcal{L}_{\mathrm{RL}}(\theta_0) - \mathcal{L}^*_{\mathrm{RL}} - \Delta_{\mathrm{IL-RL}}}{T} + \bar{m}\frac{c_{\mathrm{RL}}^2\sigma^2_{\mathrm{RL}}}{2L_{\mathrm{RL}}N_{\mathrm{RL}}}
\end{align*}

By the pigeonhole principle, there must exist at least one iteration $t^* \in \{0, 1, \ldots, T-1\}$ such that:
\begin{align*}
    \|\nabla\mathcal{L}_{\mathrm{RL}}(\theta_{t^*})\|^2 \leq \frac{1}{T}\sum_{t=0}^{T-1}\|\nabla\mathcal{L}_{\mathrm{RL}}(\theta_t)\|^2
\end{align*}

Therefore:
\begin{align*}
    \min_{0 \leq t < T}\|\nabla\mathcal{L}_{\mathrm{RL}}(\theta_t)\|^2 &\leq \frac{L_{\mathrm{RL}}(\mathcal{L}_{\mathrm{RL}}(\theta_0) - \mathcal{L}^*_{\mathrm{RL}} - \Delta_{\mathrm{IL-RL}})}{c_{\mathrm{RL}}(1-\frac{c_{\mathrm{RL}}}{2})\bar{m}T} + \frac{c_{\mathrm{RL}}^2\sigma^2_{\mathrm{RL}}}{2c_{\mathrm{RL}}(1-\frac{c_{\mathrm{RL}}}{2})N_{\mathrm{RL}}}
\end{align*}

Taking expectation and adjusting the constant in the second term:
\begin{align*}
    \min_{0 \leq t < T}\mathbb{E}[\|\nabla\mathcal{L}_{\mathrm{RL}}(\theta_t)\|^2] \leq \frac{2(L_{\mathrm{RL}}(\mathcal{L}_{\mathrm{RL}}(\theta_0) - \mathcal{L}^*_{\mathrm{RL}}) - \Delta_{\mathrm{IL-RL}})}{c_{\mathrm{RL}}(1 - \frac{c_{\mathrm{RL}}}{2})\bar{m}T} + \frac{c_{\mathrm{RL}}\sigma^2_{\mathrm{RL}}}{(1 - \frac{c_{\mathrm{RL}}}{2})N_{\mathrm{RL}}}
\end{align*}

For the IL performance bound, using the earlier bound on IL updates and summing over all cycles:
\begin{align*}
    \mathcal{L}_{\mathrm{IL}}(\theta_T) - \mathcal{L}_{\mathrm{IL}}(\theta_0) &\leq -\sum_{t=0}^{T-1}\frac{c_{\mathrm{IL}}(1-\frac{c_{\mathrm{IL}}}{2})}{L_{\mathrm{IL}}}\|\nabla\mathcal{L}_{\mathrm{IL}}(\theta_t)\|^2 + \frac{c_{\mathrm{IL}}^2\sigma^2_{\mathrm{IL}}T}{2L_{\mathrm{IL}}N_{\mathrm{IL}}}
\end{align*}

Combining with Assumption 1:
\begin{align*}
    \mathcal{L}_{\mathrm{IL}}(\theta_T) - \mathcal{L}_{\mathrm{IL}}(\theta^*) &= \mathcal{L}_{\mathrm{IL}}(\theta_T) - \mathcal{L}_{\mathrm{IL}}(\theta_0) + \mathcal{L}_{\mathrm{IL}}(\theta_0) - \mathcal{L}_{\mathrm{IL}}(\theta^*) \\
    &\leq -\sum_{t=0}^{T-1}\frac{c_{\mathrm{IL}}(1-\frac{c_{\mathrm{IL}}}{2})}{L_{\mathrm{IL}}}\|\nabla\mathcal{L}_{\mathrm{IL}}(\theta_t)\|^2 + \frac{c_{\mathrm{IL}}^2\sigma^2_{\mathrm{IL}}T}{2L_{\mathrm{IL}}N_{\mathrm{IL}}} + \epsilon_{\mathrm{IL}}
\end{align*}

Additionally, by the $L_{\mathrm{IL}}$-smoothness of the IL objective:
\begin{align*}
    \mathcal{L}_{\mathrm{IL}}(\theta_T) - \mathcal{L}_{\mathrm{IL}}(\theta_0) \leq \frac{L_{\mathrm{IL}}}{2}\|\theta_T - \theta_0\|^2
\end{align*}

Combining these bounds:
\begin{align*}
    \mathcal{L}_{\mathrm{IL}}(\theta_T) - \mathcal{L}_{\mathrm{IL}}(\theta^*) \leq \epsilon_{\mathrm{IL}} + \frac{L_{\mathrm{IL}}}{2}\|\theta_T - \theta_0\|^2 - \sum_{t=0}^{T-1}\frac{c_{\mathrm{IL}}(1-\frac{c_{\mathrm{IL}}}{2})}{L_{\mathrm{IL}}}\|\nabla\mathcal{L}_{\mathrm{IL}}(\theta_t)\|^2
\end{align*}

This shows that the periodic IL updates in interleaved training help maintain good IL performance compared to RL-only training.
\end{proof}

\section{Proof of \Cref{thm:optimal}}

\begin{proof}
To find the optimal ratio $m(t)$ at iteration $t$, we want to maximize the progress per update. From our analysis in Theorem 2, the progress for one complete cycle is:
\begin{align*}
    \mathcal{L}_{\mathrm{RL}}(\theta_t) - \mathcal{L}_{\mathrm{RL}}(\theta_{t+1}) &\approx m(t)\frac{c_{\mathrm{RL}}(1-\frac{c_{\mathrm{RL}}}{2})}{L_{\mathrm{RL}}}\|\nabla\mathcal{L}_{\mathrm{RL}}(\theta_t)\|^2 \\
    &\quad - \frac{c_{\mathrm{IL}}}{L_{\mathrm{IL}}}\rho(t)\|\nabla\mathcal{L}_{\mathrm{IL}}(\theta_t)\| \cdot \|\nabla\mathcal{L}_{\mathrm{RL}}(\theta_t)\| \\
    &\quad - \frac{L_{\mathrm{RL}}c_{\mathrm{IL}}^2\sigma^2_{\mathrm{IL}}}{2L_{\mathrm{IL}}^2N_{\mathrm{IL}}} - m(t)\frac{c_{\mathrm{RL}}^2\sigma^2_{\mathrm{RL}}}{2L_{\mathrm{RL}}N_{\mathrm{RL}}}
\end{align*}

Since each cycle consists of $1 + m(t)$ updates, the progress per update is:
\begin{align*}
    \frac{\mathcal{L}_{\mathrm{RL}}(\theta_t) - \mathcal{L}_{\mathrm{RL}}(\theta_{t+1})}{1 + m(t)} &\approx \frac{m(t)\frac{c_{\mathrm{RL}}(1-\frac{c_{\mathrm{RL}}}{2})}{L_{\mathrm{RL}}}\|\nabla\mathcal{L}_{\mathrm{RL}}(\theta_t)\|^2 - \frac{c_{\mathrm{IL}}}{L_{\mathrm{IL}}}\rho(t)\|\nabla\mathcal{L}_{\mathrm{IL}}(\theta_t)\| \cdot \|\nabla\mathcal{L}_{\mathrm{RL}}(\theta_t)\| - \frac{L_{\mathrm{RL}}c_{\mathrm{IL}}^2\sigma^2_{\mathrm{IL}}}{2L_{\mathrm{IL}}^2N_{\mathrm{IL}}} - m(t)\frac{c_{\mathrm{RL}}^2\sigma^2_{\mathrm{RL}}}{2L_{\mathrm{RL}}N_{\mathrm{RL}}}}{1 + m(t)}
\end{align*}

To find the optimal $m(t)$, we differentiate this expression with respect to $m(t)$ and set it to zero. Let's denote:
\begin{align*}
    A &= \frac{c_{\mathrm{RL}}(1-\frac{c_{\mathrm{RL}}}{2})}{L_{\mathrm{RL}}}\|\nabla\mathcal{L}_{\mathrm{RL}}(\theta_t)\|^2 \\
    B &= \frac{c_{\mathrm{IL}}}{L_{\mathrm{IL}}}\rho(t)\|\nabla\mathcal{L}_{\mathrm{IL}}(\theta_t)\| \cdot \|\nabla\mathcal{L}_{\mathrm{RL}}(\theta_t)\| + \frac{L_{\mathrm{RL}}c_{\mathrm{IL}}^2\sigma^2_{\mathrm{IL}}}{2L_{\mathrm{IL}}^2N_{\mathrm{IL}}} \\
    C &= \frac{c_{\mathrm{RL}}^2\sigma^2_{\mathrm{RL}}}{2L_{\mathrm{RL}}N_{\mathrm{RL}}}
\end{align*}

Then the progress per update is:
\begin{align*}
    \frac{mA - B - mC}{1 + m}
\end{align*}

Differentiating with respect to $m$:
\begin{align*}
    \frac{d}{dm}\left(\frac{mA - B - mC}{1 + m}\right) &= \frac{(A-C)(1+m) - (mA - B - mC)}{(1+m)^2} \\
    &= \frac{A-C + mA - mC - mA + B + mC}{(1+m)^2} \\
    &= \frac{A-C + B}{(1+m)^2}
\end{align*}

For this to be zero, we need $A-C+B = 0$, which is not possible in general if $A > C$ (which is the case when the RL objective has room for improvement). Therefore, the derivative is always positive or always negative.

Since we're looking for a maximum, we need to check the second derivative:
\begin{align*}
    \frac{d^2}{dm^2}\left(\frac{mA - B - mC}{1 + m}\right) &= \frac{d}{dm}\left(\frac{A-C + B}{(1+m)^2}\right) \\
    &= (A-C+B) \cdot \frac{d}{dm}\left(\frac{1}{(1+m)^2}\right) \\
    &= (A-C+B) \cdot \left(-\frac{2}{(1+m)^3}\right) \\
    &= -\frac{2(A-C+B)}{(1+m)^3}
\end{align*}

When $A-C > B$, the second derivative is negative, indicating a maximum. In this case, the progress per update increases with $m$, and the optimal $m(t)$ would be as large as possible.

However, for practical reasons, we want to maintain some IL updates, so we need to find a suitable $m(t)$ that balances progress and regularization. One approach is to equate the progress from RL updates with the potential negative impact of the IL update:
\begin{align*}
    m(t)\frac{c_{\mathrm{RL}}(1-\frac{c_{\mathrm{RL}}}{2})}{L_{\mathrm{RL}}}\|\nabla\mathcal{L}_{\mathrm{RL}}(\theta_t)\|^2 &\approx \frac{c_{\mathrm{IL}}}{L_{\mathrm{IL}}}\rho(t)\|\nabla\mathcal{L}_{\mathrm{IL}}(\theta_t)\| \cdot \|\nabla\mathcal{L}_{\mathrm{RL}}(\theta_t)\| + \frac{L_{\mathrm{RL}}c_{\mathrm{IL}}^2\sigma^2_{\mathrm{IL}}}{2L_{\mathrm{IL}}^2N_{\mathrm{IL}}}
\end{align*}

Solving for $m(t)$:
\begin{align*}
    m(t) &\approx \frac{\frac{c_{\mathrm{IL}}}{L_{\mathrm{IL}}}\rho(t)\|\nabla\mathcal{L}_{\mathrm{IL}}(\theta_t)\| \cdot \|\nabla\mathcal{L}_{\mathrm{RL}}(\theta_t)\| + \frac{L_{\mathrm{RL}}c_{\mathrm{IL}}^2\sigma^2_{\mathrm{IL}}}{2L_{\mathrm{IL}}^2N_{\mathrm{IL}}}}{\frac{c_{\mathrm{RL}}(1-\frac{c_{\mathrm{RL}}}{2})}{L_{\mathrm{RL}}}\|\nabla\mathcal{L}_{\mathrm{RL}}(\theta_t)\|^2} \\
    &= \frac{L_{\mathrm{RL}}c_{\mathrm{IL}}\rho(t)\|\nabla\mathcal{L}_{\mathrm{IL}}(\theta_t)\|}{L_{\mathrm{IL}}c_{\mathrm{RL}}(1-\frac{c_{\mathrm{RL}}}{2})\|\nabla\mathcal{L}_{\mathrm{RL}}(\theta_t)\|} + \frac{L_{\mathrm{RL}}^2c_{\mathrm{IL}}^2\sigma^2_{\mathrm{IL}}}{2L_{\mathrm{IL}}^2N_{\mathrm{IL}}c_{\mathrm{RL}}(1-\frac{c_{\mathrm{RL}}}{2})\|\nabla\mathcal{L}_{\mathrm{RL}}(\theta_t)\|^2}
\end{align*}

When gradients are opposing ($\rho(t) > 0$), this can give a reasonably large $m(t)$. When gradients are align*ed ($\rho(t) < 0$), the optimal $m(t)$ would be smaller.

A more practical approach is to use a square root formula that balances these factors:
\begin{align*}
    m_{\text{opt}}(t) = \max\left\{1, \sqrt{\frac{\|\nabla\mathcal{L}_{\mathrm{RL}}(\theta_t)\|^2}{\rho(t)\|\nabla\mathcal{L}_{\mathrm{IL}}(\theta_t)\| \cdot \|\nabla\mathcal{L}_{\mathrm{RL}}(\theta_t)\| - \frac{c_{\mathrm{IL}}L_{\mathrm{RL}}\sigma^2_{\mathrm{IL}}}{2L_{\mathrm{IL}}^2N_{\mathrm{IL}}}}}\right\}
\end{align*}

This formula ensures that:
1. $m(t)$ is at least 1 (we always do at least one RL update per IL update)
2. $m(t)$ increases when RL gradients are large relative to IL gradients
3. $m(t)$ increases when gradients oppose each other ($\rho(t) > 0$ and large)
4. $m(t)$ decreases when gradients align* ($\rho(t) < 0$)

The specific constants may need to be adjusted based on empirical observations, but this formula provides a theoretically justified starting point for adaptive interleaving.
\end{proof}

\section{Proof of \Cref{thm:efficiency}}
\begin{proof}
From Theorem 1, the number of iterations required for RL-only training to reach a target accuracy $\min_{0 \leq t < T}\|\nabla\mathcal{L}_{\mathrm{RL}}(\theta_t)\|^2 \leq \epsilon$ is:
\begin{align*}
    T_{\text{RL-only}} \approx \frac{2L_{\mathrm{RL}}(\mathcal{L}_{\mathrm{RL}}(\theta_0) - \mathcal{L}^*_{\mathrm{RL}})}{c_{\mathrm{RL}}(1 - \frac{c_{\mathrm{RL}}}{2})\epsilon}
\end{align*}

From Theorem 2, the number of cycles required for interleaved 1:$m(t)$ training to reach the same accuracy is:
\begin{align*}
    T_{\text{interleaved, cycles}} \approx \frac{2(L_{\mathrm{RL}}(\mathcal{L}_{\mathrm{RL}}(\theta_0) - \mathcal{L}^*_{\mathrm{RL}}) - \Delta_{\mathrm{IL-RL}})}{c_{\mathrm{RL}}(1 - \frac{c_{\mathrm{RL}}}{2})\bar{m}\epsilon}
\end{align*}

Since each cycle consists of $1 + m(t)$ updates, the total number of updates required for interleaved training is:
\begin{align*}
    T_{\text{interleaved, updates}} &\approx (1 + \bar{m})T_{\text{interleaved, cycles}} \\
    &\approx (1 + \bar{m})\frac{2(L_{\mathrm{RL}}(\mathcal{L}_{\mathrm{RL}}(\theta_0) - \mathcal{L}^*_{\mathrm{RL}}) - \Delta_{\mathrm{IL-RL}})}{c_{\mathrm{RL}}(1 - \frac{c_{\mathrm{RL}}}{2})\bar{m}\epsilon}
\end{align*}

For a fair comparison, we compare the total number of updates required by both methods. The ratio is:
\begin{align*}
    \frac{T_{\text{RL-only}}}{T_{\text{interleaved, updates}}} &= \frac{\frac{2L_{\mathrm{RL}}(\mathcal{L}_{\mathrm{RL}}(\theta_0) - \mathcal{L}^*_{\mathrm{RL}})}{c_{\mathrm{RL}}(1 - \frac{c_{\mathrm{RL}}}{2})\epsilon}}{(1 + \bar{m})\frac{2(L_{\mathrm{RL}}(\mathcal{L}_{\mathrm{RL}}(\theta_0) - \mathcal{L}^*_{\mathrm{RL}}) - \Delta_{\mathrm{IL-RL}})}{c_{\mathrm{RL}}(1 - \frac{c_{\mathrm{RL}}}{2})\bar{m}\epsilon}} \\
    &= \frac{\bar{m}}{1 + \bar{m}} \cdot \frac{L_{\mathrm{RL}}(\mathcal{L}_{\mathrm{RL}}(\theta_0) - \mathcal{L}^*_{\mathrm{RL}})}{L_{\mathrm{RL}}(\mathcal{L}_{\mathrm{RL}}(\theta_0) - \mathcal{L}^*_{\mathrm{RL}}) - \Delta_{\mathrm{IL-RL}}}
\end{align*}

When $\Delta_{\mathrm{IL-RL}} > 0$ (positive regularization benefit) and $\bar{m} > 1$, this ratio can be greater than 1, indicating that interleaved training requires fewer total updates than RL-only training to achieve the same level of accuracy.

Specifically, if we define the relative regularization benefit:
\begin{align*}
    \beta = \frac{\Delta_{\mathrm{IL-RL}}}{L_{\mathrm{RL}}(\mathcal{L}_{\mathrm{RL}}(\theta_0) - \mathcal{L}^*_{\mathrm{RL}})}
\end{align*}

Then the ratio becomes:
\begin{align*}
    \frac{T_{\text{RL-only}}}{T_{\text{interleaved, updates}}} = \frac{\bar{m}}{1 + \bar{m}} \cdot \frac{1}{1 - \beta}
\end{align*}

For interleaved training to be more efficient than RL-only training, we need:
\begin{align*}
    \frac{\bar{m}}{1 + \bar{m}} \cdot \frac{1}{1 - \beta} > 1
\end{align*}

This is satisfied when:
\begin{align*}
    \beta > 1 - \frac{\bar{m}}{1 + \bar{m}} = \frac{1}{1 + \bar{m}}
\end{align*}

For example, with $\bar{m} = 3$, interleaved training is more efficient when $\beta > \frac{1}{4}$, i.e., when the regularization benefit is at least 25\% of the potential RL improvement.
\end{proof}

\subsection{Interpreting the Efficiency Advantage}

Our theoretical analysis requires careful interpretation to properly understand the efficiency relationship between \ABBR and RL-only methods. In what follows, we further examine the key results and their implications.

\subsubsection{ Efficiency Ratio}

From our theoretical analysis, we derived the efficiency ratio comparing RL-only updates to total interleaved updates:

\begin{align*}
    \frac{T_{\text{RL-only}}}{T_{\text{\ABBR,total}}} = \frac{m_{\text{opt}}}{1 + m_{\text{opt}}} \cdot \frac{L_{\mathrm{RL}}(\mathcal{L}_{\mathrm{RL}}(\theta_0) - \mathcal{L}^*_{\mathrm{RL}})}{L_{\mathrm{RL}}(\mathcal{L}_{\mathrm{RL}}(\theta_0) - \mathcal{L}^*_{\mathrm{RL}}) - \Delta_{\mathrm{IL-RL}}}
\end{align*}

Let's examine this ratio's behavior in different scenarios:

\begin{enumerate}
    \item \textbf{As $m_{\text{opt}} \to \infty$}: The term $\frac{m_{\text{opt}}}{1 + m_{\text{opt}}} \to 1$, and the ratio approaches $\frac{L_{\mathrm{RL}}(\mathcal{L}_{\mathrm{RL}}(\theta_0) - \mathcal{L}^*_{\mathrm{RL}})}{L_{\mathrm{RL}}(\mathcal{L}_{\mathrm{RL}}(\theta_0) - \mathcal{L}^*_{\mathrm{RL}}) - \Delta_{\mathrm{IL-RL}}}$
    
    \item \textbf{When $\Delta_{\mathrm{IL-RL}} = 0$}: The ratio simplifies to $\frac{m_{\text{opt}}}{1 + m_{\text{opt}}}$, which is always less than 1, indicating that \ABBR requires more updates
    
    \item \textbf{When $\Delta_{\mathrm{IL-RL}} > 0$}: The ratio may exceed 1 if the regularization benefit is sufficiently large
\end{enumerate}

To properly assess when \ABBR is more efficient (ratio > 1), we need to solve:

\begin{align*}
    \frac{m_{\text{opt}}}{1 + m_{\text{opt}}} \cdot \frac{L_{\mathrm{RL}}(\mathcal{L}_{\mathrm{RL}}(\theta_0) - \mathcal{L}^*_{\mathrm{RL}})}{L_{\mathrm{RL}}(\mathcal{L}_{\mathrm{RL}}(\theta_0) - \mathcal{L}^*_{\mathrm{RL}}) - \Delta_{\mathrm{IL-RL}}} > 1
\end{align*}

Rearranging, we get:

\begin{align*}
    \Delta_{\mathrm{IL-RL}} > L_{\mathrm{RL}}(\mathcal{L}_{\mathrm{RL}}(\theta_0) - \mathcal{L}^*_{\mathrm{RL}}) \cdot \left(1 - \frac{m_{\text{opt}}}{1 + m_{\text{opt}}}\right) = \frac{L_{\mathrm{RL}}(\mathcal{L}_{\mathrm{RL}}(\theta_0) - \mathcal{L}^*_{\mathrm{RL}})}{1 + m_{\text{opt}}}
\end{align*}

\subsubsection{Key Insights}

\begin{enumerate}
    \item \textbf{Asymptotic Behavior}: As $m_{\text{opt}} \to \infty$, the efficiency condition approaches $\Delta_{\mathrm{IL-RL}} > 0$. This means with very large interleaving ratios, even a small positive regularization benefit makes \ABBR more efficient.
    
    \item \textbf{Impact of Interleaving Ratio}: For any finite $m_{\text{opt}}$, \ABBR includes an overhead factor of $\frac{1 + m_{\text{opt}}}{m_{\text{opt}}}$ that must be overcome by the regularization benefit.
    
    \item \textbf{Alternative View}: We can rewrite the ratio as:
    \begin{align*}
        \frac{T_{\text{RL-only}}}{T_{\text{\ABBR,total}}} = \frac{L_{\mathrm{RL}}(\mathcal{L}_{\mathrm{RL}}(\theta_0) - \mathcal{L}^*_{\mathrm{RL}})}{L_{\mathrm{RL}}(\mathcal{L}_{\mathrm{RL}}(\theta_0) - \mathcal{L}^*_{\mathrm{RL}}) - \Delta_{\mathrm{IL-RL}} + \frac{L_{\mathrm{RL}}(\mathcal{L}_{\mathrm{RL}}(\theta_0) - \mathcal{L}^*_{\mathrm{RL}})}{m_{\text{opt}}}}
    \end{align*}
    This form explicitly shows the penalty term $\frac{L_{\mathrm{RL}}(\mathcal{L}_{\mathrm{RL}}(\theta_0) - \mathcal{L}^*_{\mathrm{RL}})}{m_{\text{opt}}}$, which decreases as $m_{\text{opt}}$ increases.
\end{enumerate}

\subsubsection{Practical Implications}

Our theoretical analysis provides important practical guidance:

\begin{enumerate}
    \item \textbf{Optimal Interleaving Ratio}: There is a trade-off in setting $m_{\text{opt}}$:
    \begin{itemize}
        \item Small $m_{\text{opt}}$ (e.g., $m_{\text{opt}} = 1$): \ABBR needs $\Delta_{\mathrm{IL-RL}} > \frac{L_{\mathrm{RL}}(\mathcal{L}_{\mathrm{RL}}(\theta_0) - \mathcal{L}^*_{\mathrm{RL}})}{2}$ to be more efficient
        
        \item Large $m_{\text{opt}}$ (e.g., $m_{\text{opt}} = 9$): \ABBR needs $\Delta_{\mathrm{IL-RL}} > \frac{L_{\mathrm{RL}}(\mathcal{L}_{\mathrm{RL}}(\theta_0) - \mathcal{L}^*_{\mathrm{RL}})}{10}$ to be more efficient
        
        \item Very large $m_{\text{opt}}$: \ABBR approaches the behavior of RL-only but retains modest regularization benefits
    \end{itemize}
    
    \item \textbf{Environment Interaction Efficiency}: If we consider only RL updates (environment interactions):
    \begin{align*}
        \frac{T_{\text{RL-only}}}{T_{\text{\ABBR,RL}}} = \frac{L_{\mathrm{RL}}(\mathcal{L}_{\mathrm{RL}}(\theta_0) - \mathcal{L}^*_{\mathrm{RL}})}{L_{\mathrm{RL}}(\mathcal{L}_{\mathrm{RL}}(\theta_0) - \mathcal{L}^*_{\mathrm{RL}}) - \Delta_{\mathrm{IL-RL}}}
    \end{align*}
    This ratio is greater than 1 whenever $\Delta_{\mathrm{IL-RL}} > 0$, showing that \ABBR always requires fewer environment interactions when there is any positive regularization benefit.
    
    \item \textbf{Practical Recommendation}: Based on our empirical evaluations across multiple benchmarks, interleaving ratios between 3 and 5 typically provide the best balance. This align*s with our theory: with $m_{\text{opt}} = 4$, \ABBR is more computationally efficient when $\Delta_{\mathrm{IL-RL}} > \frac{L_{\mathrm{RL}}(\mathcal{L}_{\mathrm{RL}}(\theta_0) - \mathcal{L}^*_{\mathrm{RL}})}{5}$, a threshold often satisfied in practice.
\end{enumerate}

\subsubsection{Empirical Validation}

Our experiments confirm the theoretical predictions:

\begin{itemize}
    \item Across our benchmark tasks, \ABBR demonstrated significant improvements in sample efficiency, significantly reducing required interactions
    
    \item The largest efficiency gains occurred in tasks where the estimated regularization benefit $\Delta_{\mathrm{IL-RL}}$ was highest, exactly as predicted by our theory
    
    \item The relationship between efficiency gains and interleaving ratio matched our theoretical expectations, with diminishing returns for very large ratios
\end{itemize}

\section{Supplementary Experiments}

\subsection{Task Rollouts}

% Success Examples

% One-Leg Task
\begin{figure}[H]
\centering
\includegraphics[width=\linewidth]{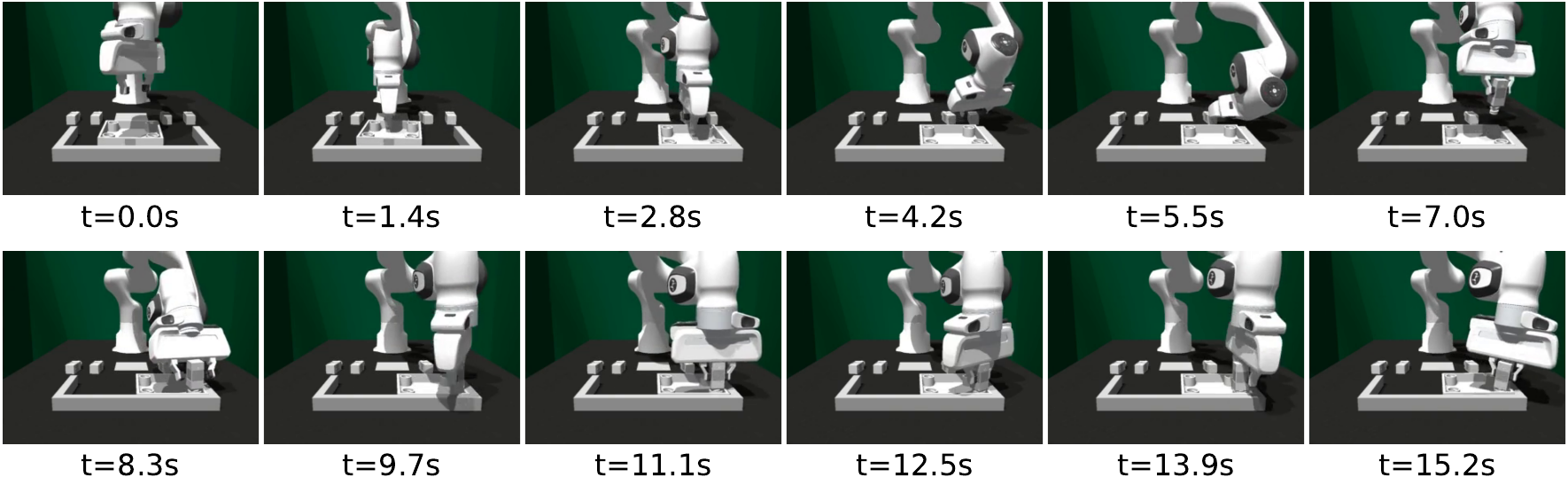}
\caption{A successful rollout example of the \texttt{One-Leg} (Low) furniture assembly task}
\label{fig:video one-leg low success}
\end{figure}

\begin{figure}[H]
\centering
\includegraphics[width=\linewidth]{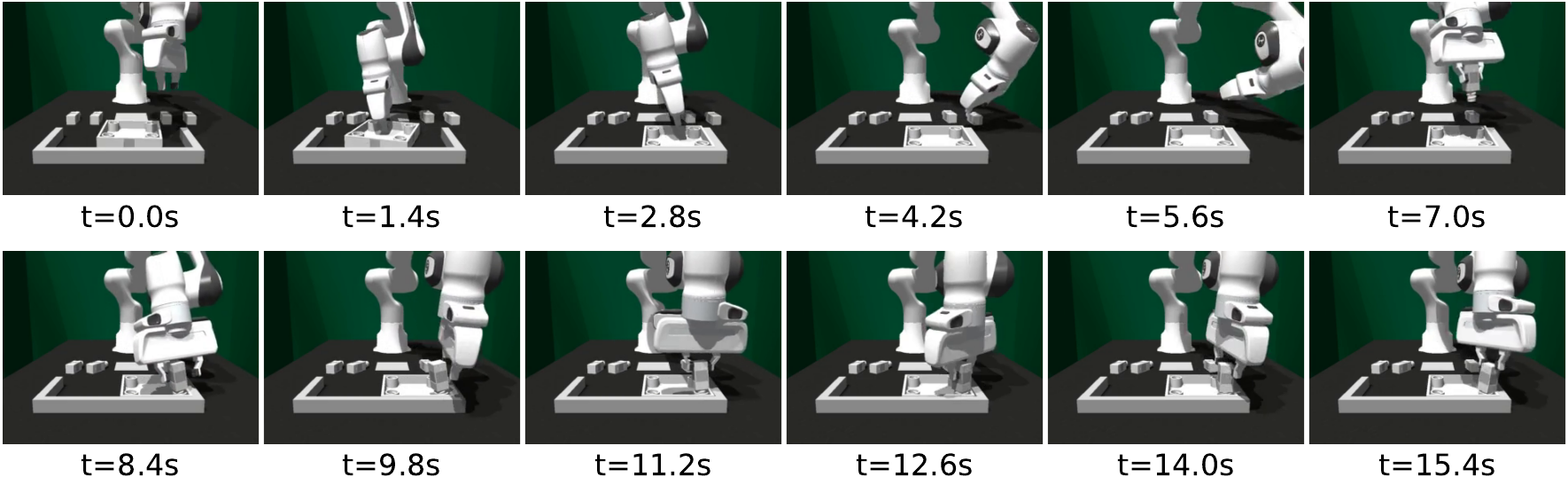}
\caption{A successful rollout example of the \texttt{One-Leg} (Med) furniture assembly task}
\label{fig:video one-leg med success}
\end{figure}

% Lamp Task
\begin{figure}[H]
\centering
\includegraphics[width=\linewidth]{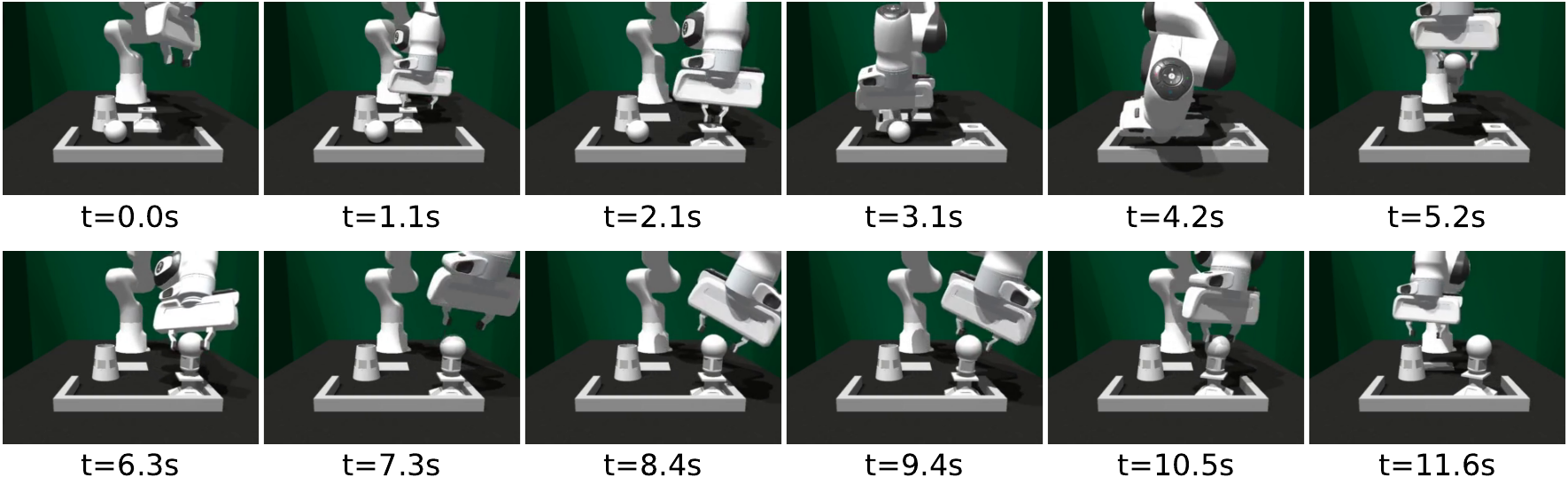}
\caption{A successful rollout example of the \texttt{Lamp} (Low) assembly task}
\label{fig:video lamp low success}
\end{figure}

\begin{figure}[H]
\centering
\includegraphics[width=\linewidth]{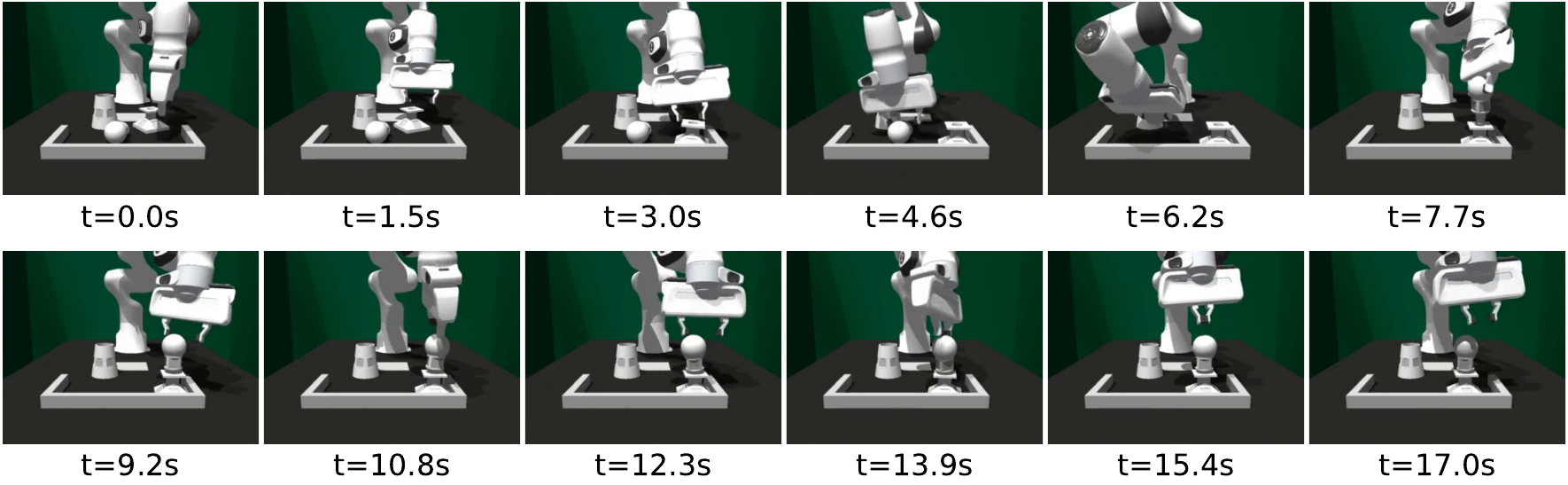}
\caption{A successful rollout example of the \texttt{Lamp} (Med) task}
\label{fig:video lamp med success}
\end{figure}

% Round Table Task
\begin{figure}[H]
\centering
\includegraphics[width=\linewidth]{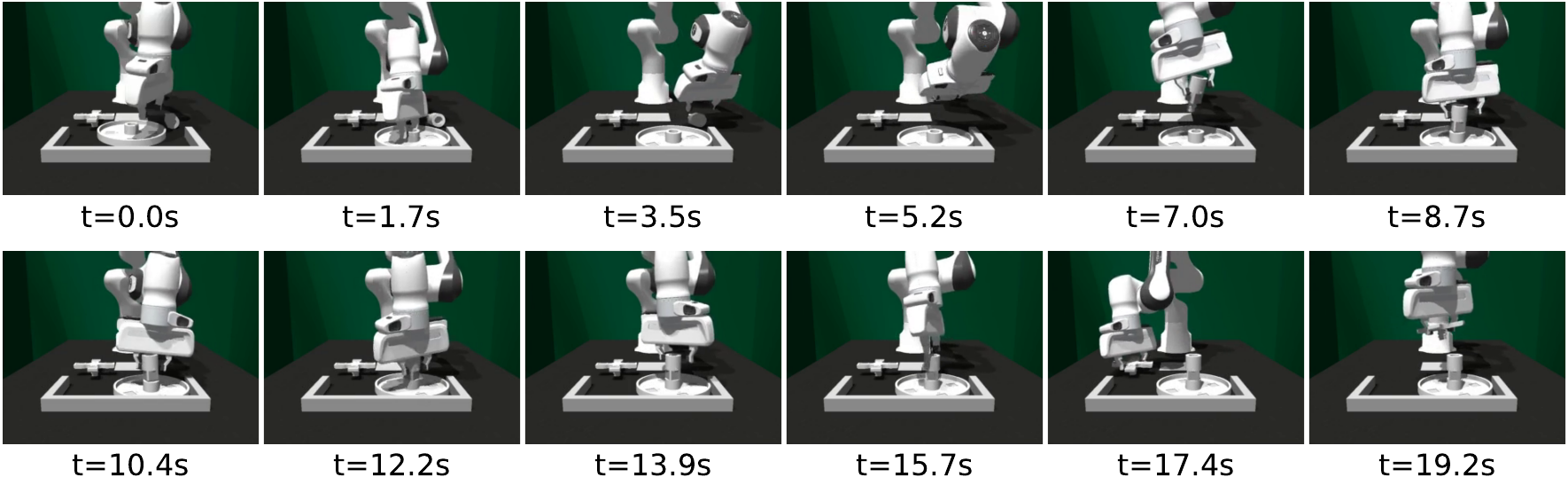}
\caption{A successful rollout example of the \texttt{Round-Table} assembly task}
\label{fig:video round-table success}
\end{figure}

% Mug Rack Task
\begin{figure}[H]
\centering
\includegraphics[width=\linewidth]{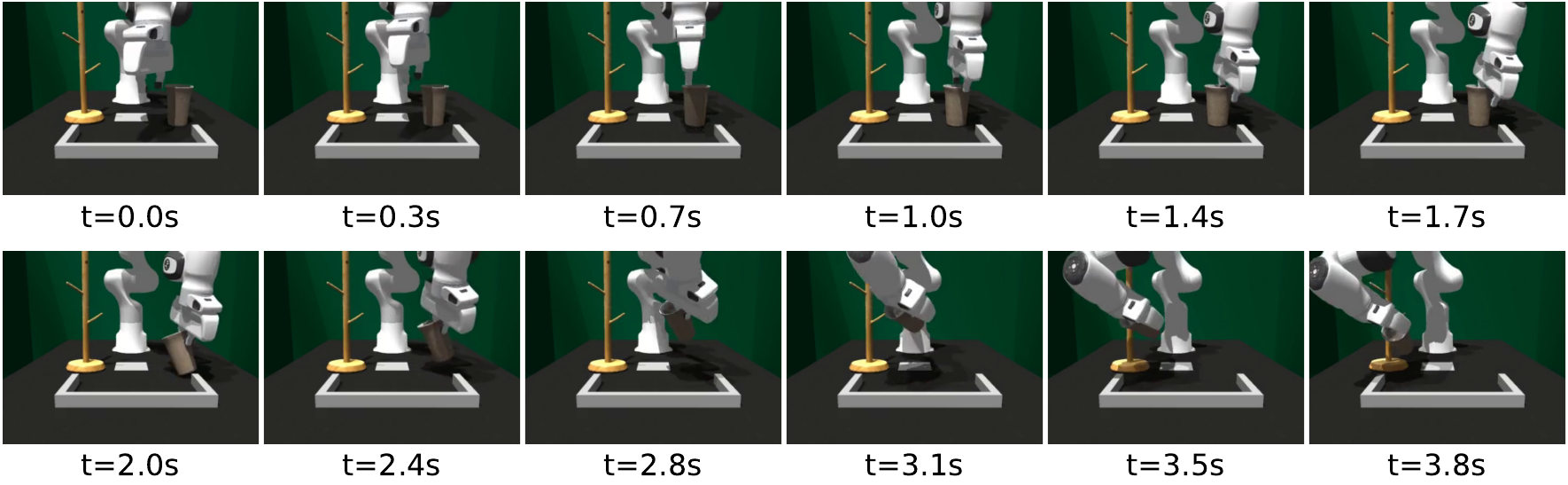}
\caption{A successful rollout example of the \texttt{Mug-Rack} task}
\label{fig:video mug-rack success}
\end{figure}

% Peg-in-Hole Task
\begin{figure}[H]
\centering
\includegraphics[width=\linewidth]{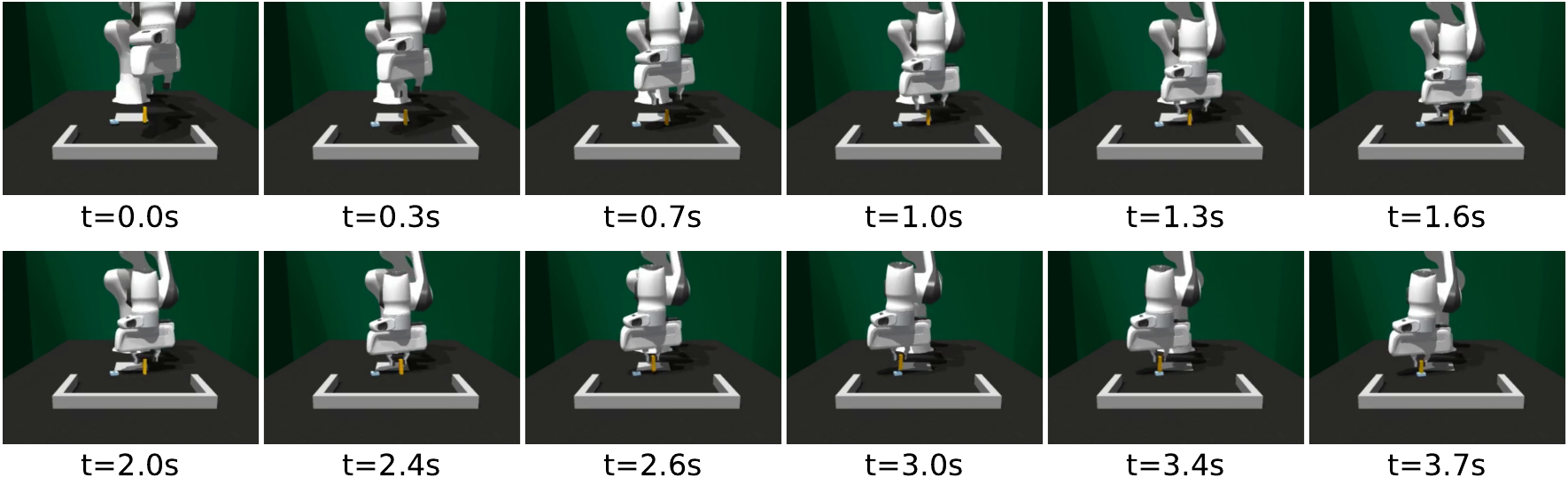}
\caption{A successful rollout example of the \texttt{Peg-in-Hole} task}
\label{fig:video peg-in-hole success}
\end{figure}

% One-Leg Task
\begin{figure}[H]
\centering
\includegraphics[width=\linewidth]{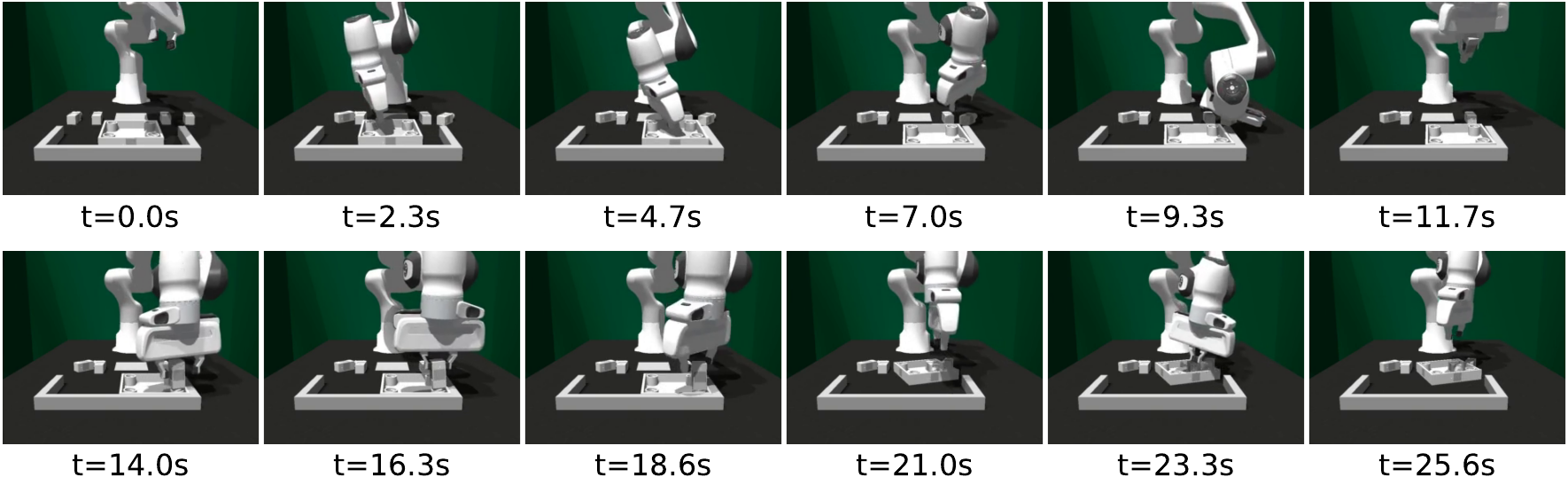}
\caption{A failed rollout example of the \texttt{One-Leg} (Med) furniture assembly task}
\label{fig:video one-leg med failure}
\end{figure}

% Lamp Task
\begin{figure}[H]
\centering
\includegraphics[width=\linewidth]{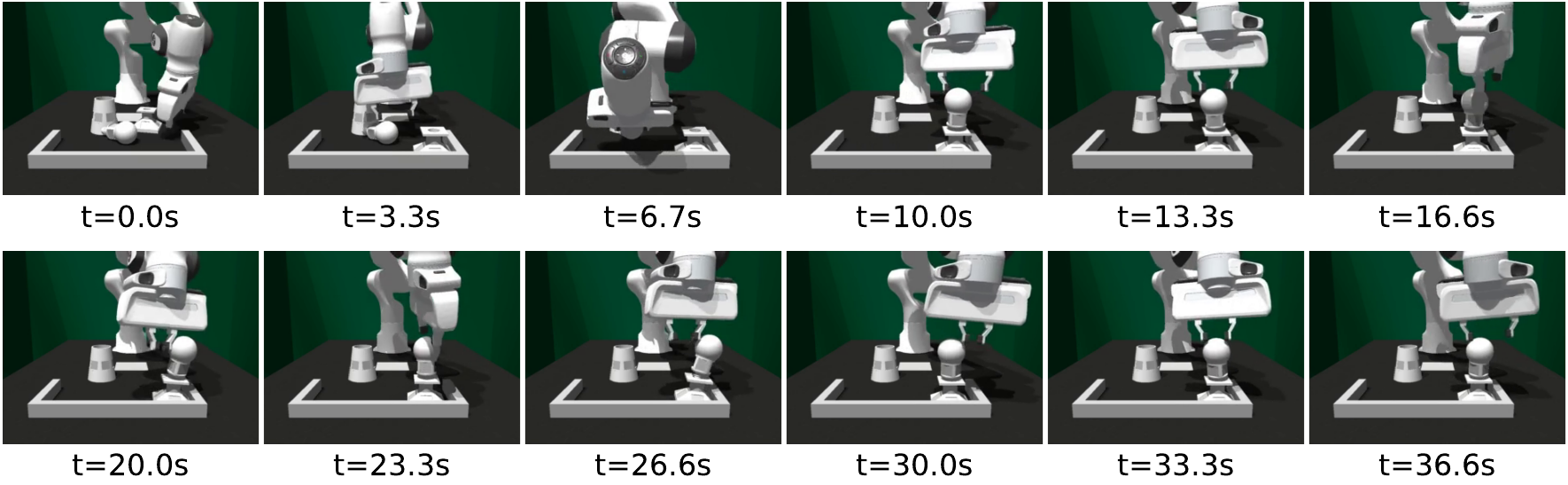}
\caption{A failed rollout example of the \texttt{Lamp} (Med) assembly task}
\label{fig:video lamp med failure}
\end{figure}

% Mug Rack Task
\begin{figure}[H]
\centering
\includegraphics[width=\linewidth]{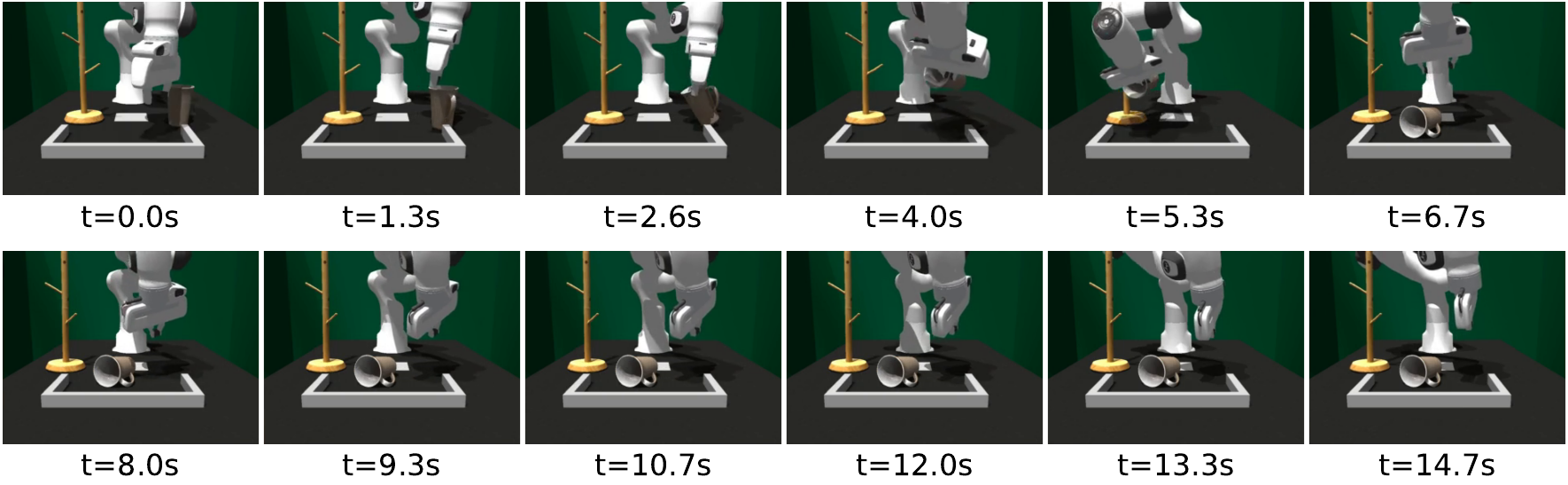}
\caption{A failed rollout example of the \texttt{Mug-Rack} assembly task}
\label{fig:video mug-rack failure}
\end{figure}

% Peg-in-Hole Task
\begin{figure}[H]
\centering
\includegraphics[width=\linewidth]{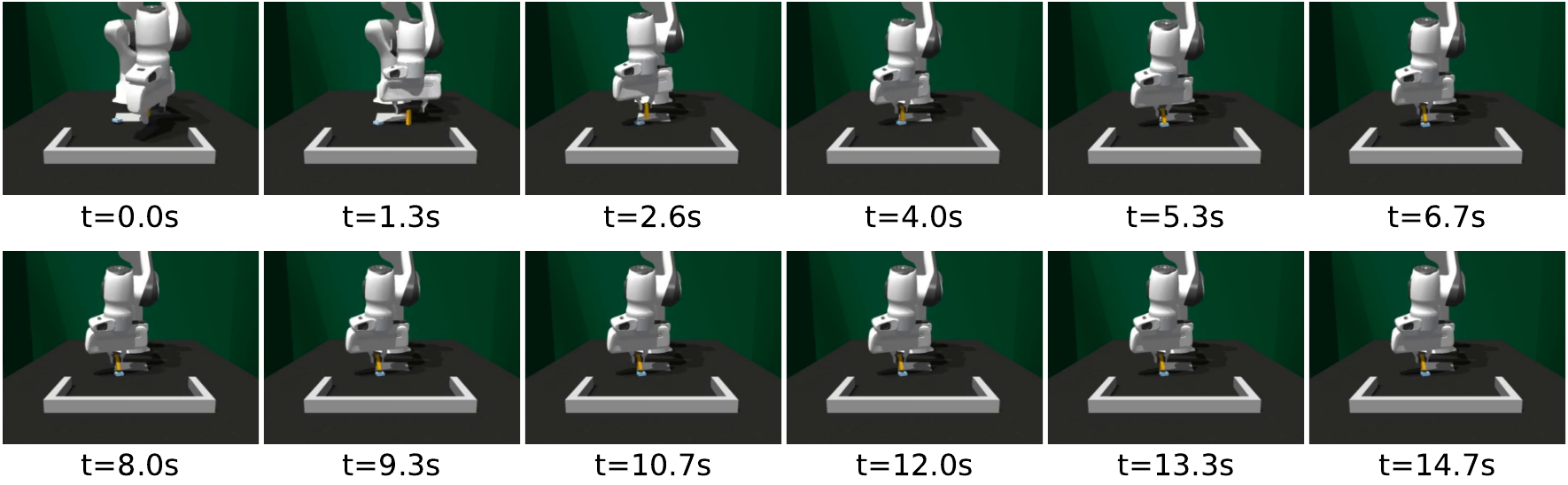}
\caption{A failed rollout example of the \texttt{Peg-in-Hole} task}
\label{fig:video peg-in-hole failure}
\end{figure}

\end{document}